\documentclass{article}

\usepackage[preprint]{neurips_2024}

\usepackage[utf8]{inputenc} 
\usepackage[T1]{fontenc}    
\usepackage{hyperref}       
\usepackage{url}            
\usepackage{booktabs}       
\usepackage{makecell}       
\usepackage{amsfonts}       
\usepackage{nicefrac}       
\usepackage{microtype}      
\usepackage{xcolor}         

\usepackage{amsmath}
\usepackage{amssymb}
\usepackage{mathtools}
\usepackage{amsthm}
\usepackage{thmtools}
\usepackage{thm-restate}

\usepackage[capitalize,noabbrev,nameinlink]{cleveref}

\usepackage[]{caption}
\usepackage{subcaption}
\usepackage{float}
\usepackage{algorithm,algpseudocode}
\usepackage{latexsym,amssymb,amsmath,amsthm,mathtools,dsfont}
\usepackage{graphicx,xcolor,graphics,url}
\usepackage{verbatim,ifthen,alltt,array}
\usepackage{multirow}
\usepackage{enumitem}
\usepackage{titlesec}
\usepackage{xparse}
\usepackage{xspace,setspace}

\usepackage{xfrac}
\usepackage[framemethod=TikZ]{mdframed}
\usepackage{nicefrac}       

\usepackage{stmaryrd}
\usepackage{textcomp}
\usepackage{siunitx}
\usepackage{xpatch}
\usepackage{xspace}

\usepackage{pgf}
\usepackage{tikz}
\usepackage{forest}

\usetikzlibrary{arrows,decorations.pathmorphing,decorations.pathreplacing,decorations.footprints,fadings,calc,trees,mindmap,shadows,decorations.text,patterns,positioning,shapes,matrix,fit}
\usetikzlibrary{arrows,shadows,backgrounds}
\usetikzlibrary{arrows.meta}
\usetikzlibrary{positioning,fit}
\usetikzlibrary{automata}
\usetikzlibrary{matrix}
\usetikzlibrary{shapes.symbols,shapes.misc,shapes.arrows}
\usetikzlibrary{matrix,chains,scopes,decorations.pathmorphing}
\usetikzlibrary{bayesnet}
\usetikzlibrary{tikzmark}

\usepackage{tcolorbox}

\theoremstyle{plain}
\newtheorem{theorem}{Theorem}[section]

\theoremstyle{definition}
\newtheorem{definition}[theorem]{Definition}

\theoremstyle{remark}

\newtheorem{example}[theorem]{Example}

\usepackage[textsize=tiny]{todonotes}

\definecolor{darkred}{rgb}{0.7,0.1,0.1}
\definecolor{medred}{rgb}{0.5,0.1,0.1}
\definecolor{midred}{rgb}{0.7,0.2,0.2}
\definecolor{vdarkred}{rgb}{0.4,0.1,0.1}
\definecolor{darkslategray}{rgb}{0.18, 0.31, 0.31} 
\definecolor{platinum}{rgb}{0.9, 0.89, 0.89} 
\definecolor{gray}{rgb}{.4,.4,.4}
\definecolor{midgrey}{rgb}{0.5,0.5,0.5}
\definecolor{middarkgrey}{rgb}{0.35,0.35,0.35}
\definecolor{darkgrey}{rgb}{0.3,0.3,0.3}
\definecolor{darkred}{rgb}{0.7,0.1,0.1}
\definecolor{midblue}{rgb}{0.2,0.2,0.7}
\definecolor{darkblue}{rgb}{0.1,0.1,0.5}
\definecolor{darkgreen}{rgb}{0.1,0.5,0.1}
\definecolor{defseagreen}{cmyk}{0.69,0,0.50,0}
\definecolor{purple3}{RGB}{125,38,205}          

\definecolor{tyellow1}{HTML}{FCE94F}
\definecolor{tyellow2}{HTML}{EDD400}
\definecolor{tyellow3}{HTML}{C4A000}
\definecolor{torange1}{HTML}{FCAF3E}
\definecolor{torange2}{HTML}{F57900}
\definecolor{torange3}{HTML}{C35C00}
\definecolor{tbrown1}{HTML}{E9B96E}
\definecolor{tbrown2}{HTML}{C17D11}
\definecolor{tbrown3}{HTML}{8F5902}
\definecolor{tgreen1}{HTML}{8AE234}
\definecolor{tgreen2}{HTML}{73D216}
\definecolor{tgreen3}{HTML}{4E9A06}
\definecolor{tblue1}{HTML}{729FCF}
\definecolor{tblue2}{HTML}{3465A4}
\definecolor{tblue3}{HTML}{204A87}
\definecolor{tpurple1}{HTML}{AD7FA8}
\definecolor{tpurple2}{HTML}{75507B}
\definecolor{tpurple3}{HTML}{5C3566}
\definecolor{tred1}{HTML}{EF2929}
\definecolor{tred2}{HTML}{CC0000}
\definecolor{tred3}{HTML}{A40000}
\definecolor{tlgray1}{HTML}{EEEEEC}
\definecolor{tlgray2}{HTML}{D3D7CF}
\definecolor{tlgray3}{HTML}{BABDB6}
\definecolor{tdgray1}{HTML}{888A85}
\definecolor{tdgray2}{HTML}{555753}
\definecolor{tdgray3}{HTML}{2E3436}

\newcommand{\dghlight}[1]{{\color[RGB]{0,120,0}#1}}

\newtheoremstyle{nthmstyle}
{3pt}
{3pt}
{}
{}
{\bfseries}
{.}
{.5em}
{}

\theoremstyle{nthmstyle}

\crefname{enumi}{}{}
\crefname{rstprop}{Proposition}{Propositions}

\newcommand{\fml}[1]{{\mathcal{#1}}}

\newcommand{\tn}[1]{\textnormal{#1}}

\newcommand{\msf}[1]{\ensuremath\mathsf{#1}}
\newcommand{\mbf}[1]{\ensuremath\mathbf{#1}}
\newcommand{\mbb}[1]{\ensuremath\mathbb{#1}}

\newcommand{\waxp}{\ensuremath\mathsf{WAXp}}
\newcommand{\wcxp}{\ensuremath\mathsf{WCXp}}
\newcommand{\axp}{\ensuremath\mathsf{AXp}}
\newcommand{\cxp}{\ensuremath\mathsf{CXp}}
\newcommand{\aex}{\ensuremath\mathsf{AEx}}

\newcommand{\xhsc}{AxFi\xspace}

\DeclareMathOperator*{\fs}{\msf{Fs}}

\newcommand{\fsn}[1]{\msf{Fs}_{#1}}

\newcommand{\bigand}{\bigwedge}

\newcommand{\exv}{\ensuremath\mathbf{E}}

\newcommand{\cf}{\ensuremath\upsilon} 
\newcommand{\cfn}[1]{\ensuremath\upsilon_{#1}} 

\newcommand{\svn}[1]{\msf{Sc}_{#1}}

\newcommand{\similar}{\ensuremath\sigma}

\DeclareMathOperator*{\sv}{\msf{Sc}}

\definecolor{gray}{rgb}{.4,.4,.4}
\definecolor{midgrey}{rgb}{0.5,0.5,0.5}
\definecolor{middarkgrey}{rgb}{0.35,0.35,0.35}
\definecolor{darkgrey}{rgb}{0.3,0.3,0.3}
\definecolor{darkred}{rgb}{0.7,0.1,0.1}
\definecolor{midblue}{rgb}{0.2,0.2,0.7}
\definecolor{darkblue}{rgb}{0.1,0.1,0.5}
\definecolor{defseagreen}{cmyk}{0.69,0,0.50,0}

\newcommand{\jnoteF}[1]{}

\newcounter{Comment}[Comment]
\DeclareMathOperator*{\limply}{\rightarrow}

\declaretheoremstyle[
  headfont=\bfseries,
  bodyfont=\itshape,
  numberwithin=section,
]{StdThmStyle}

\tikzset{
  0 my edge/.style={densely dashed, my edge},
  my edge/.style={-{Stealth[]}},
}

\setlist{nosep}

\titleformat{\paragraph}[runin]
{\normalfont\normalsize\bfseries}{\theparagraph}{1em}{}

\titlespacing{\section}{0pt}{*2.15}{*1.0}
\titlespacing{\subsection}{0pt}{*1.25}{*0.75}
\titlespacing{\subsubsection}{0pt}{*0.35}{*0.5}
\titlespacing{\paragraph}{0pt}{*0.1}{*0.575}

\makeatletter
\newcommand\nparagraph{%
  \@startsection{paragraph}
    {4}
    {\z@}
    {0.225ex \@plus0.225ex \@minus.125ex}
    {-1em}
    {\normalfont\normalsize\bfseries}%
}
\makeatother

\titlespacing{\paragraph}{%
  0pt}{
  0.225\baselineskip}{
  1em}

\title{Rigorous Feature Importance Scores based on Shapley Value and Banzhaf Index}

\author{%
    Xuanxiang Huang \\
    Nanyang Technological University \\
    CNRS@CREATE \\
    Singapore \\
    \texttt{xuanxiang.huang@ntu.edu.sg} \\
    \And
    Olivier L\'{e}toff\'{e} \\
    University of Toulouse \\
    France \\
    \texttt{olivier.letoffe@orange.fr} \\
    \And 
    Joao Marques-Silva \\
    ICREA, University of Lleida \\
    Spain \\
    \texttt{jpms@icrea.cat} \\
}

\begin{document}

\maketitle

\begin{abstract}
  Feature attribution methods based on game theory are ubiquitous in
  the field of eXplainable Artificial Intelligence (XAI).
  Recent works proposed rigorous feature attribution 
  using logic-based explanations, specifically targeting high-stakes
  uses of machine learning (ML) models.
  Typically, such works exploit \emph{weak abductive explanation}
  (WAXp) as the characteristic function to assign importance to
  features. However, one possible downside is that the contribution of
  non-WAXp sets is neglected.
  In fact, non-WAXp sets can also convey important information, because
  of the relationship between \emph{formal explanations} (XPs) and
  \emph{adversarial examples} (AExs).
  Accordingly, this paper leverages Shapley value and Banzhaf index 
  to devise two novel feature importance scores.
  We take into account non-WAXp sets when computing feature contribution,
  and the novel scores quantify how effective each feature is at excluding AExs.
  Furthermore, the paper
  identifies properties and studies the computational complexity
  of the proposed scores.
\end{abstract}

\section{Introduction} \label{sec:intro}
There is an increasing demand for transparency and accountability
in the decision-making processes of complex machine learning (ML) models
~\cite{yu2019s,novelli2023accountability}.
This demand has spurred rapid growth in the research domain of eXplainable AI (XAI)
~\cite{miller2019explanation,xai-bk19,molnar-bk20}.
XAI can be defined as the process of bridging the gap between the inner workings of 
ML systems and human understanding, with the goal of achieving trustworthy AI.
Feature attribution methods are one of the most widely used approaches
~\cite{kononenko-jmlr10,guestrin-kdd16,lundberg-nips20} in XAI.
Among these, game-theoretic attribution methods such as 
SHAP scores~\cite{kononenko-jmlr10,lundberg-nips17,barcelo-jmlr23} are highly popular.
The implementation of these methods depend on the choice of characteristic functions
and power index frameworks; different choices can lead to different outcomes
~\cite{najmi-icml20,blockbaum-aistats20}.

Within the domain of formal XAI~\cite{msi-aaai22,darwiche-lics23}, several rigorous feature attribution methods
~\cite{ignatiev-sat24,izza-aaai24}
have been proposed to ensure the rigor of explanations in high-risk and safety-critical applications.
Essentially, these methods employ logic-based predicates, most use weak abductive explanation (WAXp),
as the characteristic functions for computing feature importance.
Moreover, these rigorous feature attribution methods can be generally referred to as axiomatic 
aggregations of features~\cite{izza-aaai24},
and the computed scores can be termed
\emph{feature importance scores}.

Although logic-based predicates provide rigorous guarantees, 
these predicates are too restricted and only provide a 
coarse-grained quantification of rigorous feature importance,
limiting their ability to provide a broader picture.
To justify our claim, let us switch our angle from computing formal explanations 
(including abductive explanations (AXp) and contrastive explanations (CXp)~\cite{msi-aaai22}) to 
detecting and excluding adversarial examples (AExs)~\cite{szegedy-iclr14}
in the vicinity of the target data point, where the connection 
between AExs and formal explanations 
has been established~\cite{inms-nips19}.
Given two sets of fixed features that are non-WAXp, such WAXp predicates can
only indicate that both sets failed to exclude all AExs.
However, it is possible that one set excludes more AExs than the other,
and such information cannot be conveyed using WAXp predicates
as the characteristic functions.

We argue that by considering non-WAXp sets in the vicinity of the target data point,
one can achieve a more fine-grained quantification of rigorous feature importance 
and obtain a more detailed and informative perspective on feature importance in formal XAI.

\paragraph{Contributions.}
This paper proposes two novel rigorous feature importance scores,
one based on Shapley value and the other based on Banzhaf Index,
called \xhsc (\underline{A}dversarial e\underline{x}amples-based \underline{F}eature \underline{i}mportance).
Additionally, we define novel properties that feature importance scores should satisfy.
The computation of these scores is based on a novel characteristic function,
called \emph{contrastive explanation forest} (CXp-Forest).
Our work complements existing rigorous feature attribution methods in several aspects:
\begin{enumerate}[leftmargin=1em]
\item
In high-stakes and safety-critical scenarios,
\xhsc can offer rigorous guarantees compared to SHAP scores.
\item
For decision tree models, \xhsc can be computed in polynomial time,
making them more efficient compared to existing rigorous feature importance scores.
\end{enumerate}

\paragraph{Organization.}
The paper is organized as follows.
\cref{sec:prelim} introduces the notation and definitions used
throughout the paper, along with related work.
\cref{sec:cxpf} details the key component of this paper, CXp-Forest.
\cref{sec:score} presents the \xhsc.
It also analyses the properties of \xhsc scores, and the complexity of computing 
\xhsc scores.
\cref{sec:res} presents preliminary experimental evidence.
\cref{sec:limit} discusses the limitation and extension of our method.
\cref{sec:conc} concludes the paper and outlines future work.
Due to lack of space, some proofs are included in \cref{app:proofs}.


\section{Preliminaries} \label{sec:prelim}
\paragraph{Classification \& regression problems.}
Let $\fml{F}=\{1,\ldots,m\}$ denote a set of features.
Each feature $i\in\fml{F}$ takes values from a \emph{bounded} domain $\mbb{D}_i$.
Domains can be categorical or ordinal. If ordinal, domains can be
discrete or real-valued.
Feature space is defined by
$\mbb{F}=\mbb{D}_1\times\mbb{D}_2\times\ldots\times\mbb{D}_m$.
The notation $\mbf{x}=(x_1,\ldots,x_m)$ denotes an arbitrary point in 
feature space, where each $x_i$ is a variable taking values from
$\mbb{D}_i$. Moreover, the notation $\mbf{v}=(v_1,\ldots,v_m)$
represents a specific point in feature space, where each $v_i$ is a
constant representing one concrete value from $\mbb{D}_i$.
In the case of classification, each point in feature space is mapped to a
class taken from a set $\fml{K}=\{c_1,c_2,\ldots,c_K\}$.
Classes can also be categorical or ordinal.
In the case of regression, each point in feature space is mapped to an
ordinal value taken from a set $\mbb{K}$, e.g.\ $\mbb{K}$ could denote
$\mbb{Z}$ or $\mbb{R}$.
Therefore, a classifier $\fml{M}_{C}$ is characterized by a
non-constant \emph{classification function} $\kappa$ that maps feature
space $\mbb{F}$ into the set of classes $\fml{K}$,
i.e.\ $\kappa:\mbb{F}\to\fml{K}$.
A regression model $\fml{M}_R$ is characterized by a non-constant
\emph{regression function} $\rho$ that maps feature space $\mbb{F}$
into the set elements from $\mbb{K}$, i.e.\ $\rho:\mbb{F}\to\mbb{K}$.
When viable, we will represent an ML model $\fml{M}$ by a tupple
$(\fml{F},\mbb{F},\mbb{T},\tau)$, with $\tau:\mbb{F}\to\mbb{T}$,
without specifying whether $\fml{M}$ denotes a classification
or a regression model.
A \emph{sample} (or instance) denotes a pair $(\mbf{v},q)$, where
$\mbf{v}\in\mbb{F}$ and either $q\in\fml{K}$, with
$q=\kappa(\mbf{v})$, or $q\in\mbb{K}$, with $q=\rho(\mbf{v})$.

\paragraph{$l_0$ Norm.}
The distance between two vectors $\mbf{x}$ and $\mbf{y}$ is denoted by
$\lVert\mbf{x}-\mbf{y}\rVert$,
the $l_{0}$ norm is defined as follows~\cite{robinson-bk03}:
\begin{equation}
\lVert\mbf{x}-\mbf{y}\rVert_{0}=\sum\nolimits_{i=1}^{m}\tn{ITE}(x_i\not=y_i,1,0)
\end{equation}
where \emph{ITE} denotes the \emph{If-Then-Else} operator.
$l_{0}$ norm represents the number of different
variables (or features) between two vectors $\mbf{x}$ and $\mbf{y}$.

\paragraph{Distributions \& expected value.}
Throughout the paper, we assume \emph{product distributions}~\cite{vandenbroeck-aaai21},
where all features are independent.
The \emph{expected value} of an ML model $\tau$
is denoted by $\mbf{E}[\tau]$.

\paragraph{Explanation problems.}
An explanation problem is a tuple $\fml{E}=(\fml{M},(\mbf{v},q))$,
where $\fml{M}$ can either be a classification or a regression model,
and $(\mbf{v},q)$ is a target sample, with $\mbf{v}\in\mbb{F}$.

\paragraph{Shapley value \& Banzhaf index.}
Shapley value~\cite{shapley-ctg53,algaba2019handbook}
is recognized as one of the well-known power indices
for measuring the importance of voters in a priori voting power
~\cite{machover-psr04},
specifically the Shapley-Shubik index~\cite{shapley-apsr54}.
Another notable example is Banzhaf index~\cite{banzhaf-rlr65,patel2021game}.
Shapley value has been widely used in XAI
since they respect a number of important axioms~\cite{algaba2019handbook}.
To compute Shapley value, one needs to choose a characteristic function $\cf$,
which assigns a real value to each $\fml{S}\subseteq\fml{F}$.
Finally, let $\svn{S}:\fml{F}\to\mbb{R}$, i.e.\ the Shapley value for
feature $i$, be defined by
%
\begin{equation} \label{eq:sv}
\svn{S}(i;\fml{E}, \cf) := \sum_{\fml{S}\subseteq(\fml{F}\setminus\{i\})}
\left(\frac{\Delta_i(\fml{S};\fml{E},\cf)}{|\fml{F}|\times\binom{|\fml{F}|-1}{|\fml{S}|}}\right).
\end{equation}
And the Banzhaf index for feature $i$ is defined by
\begin{equation} \label{eq:bv}
\svn{B}(i;\fml{E},\cf) := \sum_{\fml{S}\subseteq(\fml{F}\setminus\{i\})}
\left(\frac{\Delta_i(\fml{S};\fml{E},\cf)}{2^{|\fml{F}|-1}}\right)
\end{equation}
To simplify the notation, the following definitions are used,
\begin{align}
\Delta_i(\fml{S};\fml{E},\cf) & :=
\left(\cf(\fml{S}\cup\{i\})-\cf(\fml{S})\right)
\label{eq:def:delta}
\end{align}

\paragraph{SHAP scores.}
SHAP~\cite{lundberg-nips17} is a well-known instantiation of Shapley values in the context of explainability,
it considers the expected value of the model's predictive function as the characteristic function~\cite{lundberg-nips17},
that is,
\[
\cfn{e}(\fml{S};\fml{E}) := \exv[\tau | \mbf{x}_{\fml{S}}=\mbf{v}_{\fml{S}}].
\]
Given an instance $(\mbf{v},q)$, the computed SHAP scores
of a feature $i$ can be regarded as that feature's contribution to the prediction.

\paragraph{Similarity predicate.}
Given an ML model and some input $\mbf{x}$, the output of the ML model
is \emph{distinguishable} with respect to the sample $(\mbf{v},q)$ if the
observed change in the model's output is deemed sufficient; 
otherwise it is \emph{similar}.
This is represented by a \emph{similarity} predicate (which can be
viewed as a boolean function)
$\similar:\mbb{F}\times\mbb{R}\to\{\bot,\top\}$ (where $\bot$ signifies
\emph{false}, and $\top$ signifies \emph{true}). Concretely, given $\delta\in\mbb{R}$,
$\similar(\mbf{x};\fml{E})$ holds true iff the change in the ML model
output is deemed \emph{insufficient} and so no observable difference
exists between the ML model's output for $\mbf{x}$ and $\mbf{v}$ by a factor of $\delta$.
\footnote{
Parameterization are shown after the separator
';', and will be elided when clear from the context.}
For regression problems, given a change in the input from $\mbf{v}$ to
$\mbf{x}$, the outputs are similar if,
\[
\similar(\mbf{x};\fml{E})
:= [|\rho(\mbf{x})-\rho(\mbf{v})|\le\delta],
\]
otherwise, it is distinguishable~\footnote{
Exploiting a threshold to decide whether there exists an observable
change has been used in the context of adversarial
robustness~\cite{barrett-nips23}.}.
For classification problems, similarity is defined as not changing the predicted class.
Given a change in the input from 
$\mbf{v}$ to $\mbf{x}$, the outputs are similar if,
\[
\similar(\mbf{x};\fml{E}) := [\kappa(\mbf{x})=\kappa(\mbf{v})].
\]

\paragraph{Adversarial examples.}
Adversarial examples (AExs) serve to reveal the brittleness of ML
models~\cite{szegedy-iclr14,szegedy-iclr15,johnson-sttt23}.
Given a sample $(\mbf{v},q)$, and the norm $l_0$
~\footnote{In fact, any norm~\cite{horn-bk12} can be used.},
a point $\mbf{x}\in\mbb{F}$ is an \emph{adversarial example} if the prediction
for $\mbf{x}$ is distinguishable from that for $\mbf{v}$. Formally, we
write,
\[
\aex(\mbf{x};\fml{E}) := 
\left(||\mbf{x}-\mbf{v}||_{0}\le\epsilon\right)\land
\neg\similar(\mbf{x};\fml{E}),
\]
where the $l_0$ distance between the given point $\mbf{v}$ and other
points of interest is restricted to $\epsilon>0$.
In this paper, we primarily focus on AExs
having the minimal distance measured by the norm $l_0$
around the given point $\mbf{v}$.
For simplicity, we will use the term `AExs' to specifically refer to
$l_0$-minimal AExs throughout the rest of the paper.

\paragraph{Abductive \& contrastive explanations.}
Abductive explanations (AXps) and contrastive explanations (CXps)
are examples of formal explanations for classification
problems~\cite{msi-aaai22}.
They can be generalized to encompass regression problems.

A weak abductive explanation (WAXp) denotes a set of features
$\fml{S}\subseteq\fml{F}$, such that for every point in feature space,
where only the features in $\fml{F}\setminus\fml{S}$ are allowed to change,
the ML model's output is similar to that of the given sample $(\mbf{v},q)$.
The condition for a set of features to represent a WAXp (which also
defines a corresponding predicate $\waxp$) is as follows:
\[
\forall(\mbf{x}\in\mbb{F}).\left(\bigand_{i\in\fml{S}}x_i=v_i\right)\limply \left(\similar(\mbf{x};\fml{E})\right).
\]
Moreover, an AXp is a subset-minimal WAXp, i.e.
\[
\axp(\fml{S};\fml{E}) := \waxp(\fml{S};\fml{E})\land\forall(t\in\fml{S}).\neg\waxp(\fml{S}\setminus\{t\};\fml{E}).
\]

A weak contrastive explanation (WCXp) denotes a set of features
$\fml{S}\subseteq\fml{F}$, such that there exists some point in
feature space, where only the features in $\fml{S}$ are allowed to
change, that makes the ML model output distinguishable from the given
sample $(\mbf{v},q)$.
The condition for a set of features to represent a WCXp (which also
defines a corresponding predicate $\wcxp$) is as follows:
\[
\exists(\mbf{x}\in\mbb{F}).\left(\bigand_{i\in\fml{F}\setminus\fml{S}}x_i=v_i\right)\land\left(\neg \similar(\mbf{x};\fml{E})\right)
\]
Moreover, a CXp is a subset-minimal WCXp, i.e.
\[
\cxp(\fml{S};\fml{E}) := \wcxp(\fml{S};\fml{E})\land\forall(t\in\fml{S}).\neg\wcxp(\fml{S}\setminus\{t\};\fml{E}).
\]

Given an explanation problem $\fml{E}$,
we use $\mbb{A}(\fml{E})$ (resp. $\mbb{C}(\fml{E})$) to denote its set of AXps (resp. CXps),
and $\mbb{A}_i(\fml{E})$ (resp. $\mbb{C}_i(\fml{E})$ to denote the set of AXps (resp. CXps)
where each AXp (resp. CXp) contains the target feature $i$.
The relationship between adversarial examples and 
formal explanations is well-known~\cite{inms-nips19}.

\paragraph{Feature (ir)relevancy.}
The set of features that are included in at least one (abductive) 
explanation are defined as follows:
\begin{equation}
\mathfrak{F}(\fml{E}) := \{i\in\fml{X}\,|\,\fml{X}\in2^{\fml{F}}\land\axp(\fml{X})\}
\end{equation}
where predicate $\axp(\fml{X})$ holds true iff $\fml{X}$ is an AXp.
($\mathfrak{F}(\fml{E})$ remains unchanged
if CXps are used instead of AXps~\cite{msi-aaai22}.)
Finally, a feature $i\in\fml{F}$ is \emph{irrelevant}, if $i\not\in\mathfrak{F}(\fml{E})$;
otherwise feature $i$ is \emph{relevant}.

\begin{figure}[t]
\begin{subfigure}{0.5\textwidth}
\begin{center}
\renewcommand{\arraystretch}{0.95}
\renewcommand{\tabcolsep}{0.5em}
\scalebox{0.85}{
\begin{tabular}{cccc} \toprule
$x_1$ & $x_2$ & $x_3$ & $\kappa(\mbf{x})$ 
\\ \toprule
0     & 0     & 0  & 0  \\
0     & 0     & 1  & 0  \\
0     & 0     & 2  & 1  \\
0     & 1     & 0  & 0  \\
0     & 1     & 1  & 0  \\
0     & 1     & 2  & 1  \\
1     & 0     & 0  & 0  \\
1     & 0     & 1  & 0  \\
1     & 0     & 2  & 0  \\
1     & 1     & 0  & 0  \\
1     & 1     & 1  & 0  \\
1     & 1     & 2  & 1  \\
2     & 0     & 0  & 0  \\
2     & 0     & 1  & 0  \\
2     & 0     & 2  & 1  \\
2     & 1     & 0  & 1  \\
2     & 1     & 1  & 1  \\
\textcolor{blue}{\textbf{2}}     & \textcolor{blue}{\textbf{1}}     & \textcolor{blue}{\textbf{2}}  & \textcolor{blue}{\textbf{1}}  \\
\bottomrule
\end{tabular}

}
\end{center}
\caption{Tabular representation of function $\kappa$.}
\label{fig:func01}
\end{subfigure}
\begin{subfigure}{0.5\textwidth}
%
\begin{center}
\renewcommand{\arraystretch}{0.95}
\renewcommand{\tabcolsep}{0.5em}
\scalebox{0.85}{
\begin{tabular}{ccc} \toprule
$\mbb{A}(\fml{E})$ & $\mbb{C}(\fml{E})$ & AExs of each CXp
\\ \toprule
$\{1,2\}$     & $\{1,2\}$     &  $\{(1,0,2)\}$ \\
$\{2,3\}$     & $\{1,3\}$     &  $\{(0,1,0), (0,1,1), (1,1,0), (1,1,1)\}$ \\
$\{1,3\}$     & $\{2,3\}$     &  $\{(2,0,0), (2,0,1)\}$ \\
\bottomrule
\end{tabular}

}
\end{center}
\caption{XPs of $\fml{E}=(\kappa, ((2,1,2),1))$.}
\label{fig:xp01}
\end{subfigure}
\caption{Classifier running example}
\end{figure}

\paragraph{Running example.}
We consider the following classification function
as the running example throughout the paper.
\begin{example} \label{ex:run01a}
Suppose we have a function $\kappa$ defined on $\fml{F}=\{1,2,3\}$ and $\fml{K}=\{0,1\}$
where $\mbb{D}_{1} = \{0,1,2\}$, $\mbb{D}_{2} = \{0,1\}$ and $\mbb{D}_{3} = \{0,1,2\}$.
Its tabular representation is depicted in~\cref{fig:func01}.
Suppose we are interested in the data point $(\mbf{v},c) = ((2,1,2),1)$.
For the explanation problem $\fml{E}=(\kappa, ((2,1,2),1))$,
\cref{fig:xp01} lists $\mbb{A}(\fml{E})$ and $\mbb{C}(\fml{E})$,
along with the AExs covered by each CXp.
Clearly, since all three features are included in $\mathfrak{F}(\fml{E})$,
there are no irrelevant features.
\end{example}

\subsection{Related Work}
Feature attribution methods based on game theory are extremely popular,
the most notable example is the SHAP scores~\cite{lundberg-nips17}.
In the past few years, many works propose different variants of Shapley values
or alternative methods for attributing importance to features
~\cite{friedler-nips21,najmi-icml20,sundararajan2017axiomatic,frye2020asymmetric,jordan-iclr19,watson-facct22,magazzeni-facct22,kulynych2017feature,patel2021high}.
There also exist works identifying issues with Shapley values and SHAP scores
~\cite{nguyen-ieee-access21,friedler-icml20,blockbaum-aistats20,hms-ijar24,msh-cacm24,lhms-aaai25},
and proposing solutions to mitigate these drawbacks.

In domains where ML applications are considered high-risk and safety-critical, 
formal XAI~\cite{msi-aaai22,darwiche-lics23}
has emerged to provide explanations with mathematically provable guarantees.
Recent work~\cite{ignatiev-sat24} proposes
two measures, i.e., \emph{formal feature attribution} (FFA) and \emph{weighted} FFA (WFFA),
for computing rigorous feature importance,
where the score of each feature reflects its frequency of occurrence in AXps.
Independent work~\cite{izza-aaai24} 
provides a more systematic approach
for aggregating AXps into different axiomatic frameworks.
They have also identified additional properties that FFA should respect.


\section{Novel Characteristic Function} \label{sec:cxpf}
This section introduces a novel characteristic function, called contrastive explanation forest (CXp-Forest),
which serves as the basis for computing fine-grained quantifications of rigorous feature importance.

\paragraph{Motivation.}
Given an explanation problem $\fml{E}$, for any set
$\fml{S} \subseteq \fml{F}$ of fixed features,
$\fml{S}$ is either a WAXp or a non-WAXp.
It is known that any WAXp is a hitting set~\cite{msi-aaai22} of CXps.
In contrast, a non-WAXp is not a hitting set of CXps; it intersects with some or none of the CXps.
Equivalently, a WAXp excludes all AExs,
whereas a non-WAXp excludes some or none of the AExs.
Moreover, for two features $i,j\not\in\fml{S}$,
if $\fml{S}\cup\{i\}$ excludes more AExs than $\fml{S}\cup\{j\}$,
we can consider feature $i$ to be more important than feature $j$ with respect to $\fml{S}$.
To achieve this, we define a function for each CXp to signify
whether a feature in this CXp intersects with $\fml{S}$.
For better illustration, we further represent this function as a \emph{binary decision tree} (DT)~\cite{wegener2000branching}.
Consequently, the set of all CXps can be represented as a forest.

\begin{definition}[Contrastive explanation Forests]
Given an explanation problem $\fml{E} = (\fml{M}, (\mbf{v}, q))$, 
along with its set of CXps $\mbb{C}(\fml{E})$ such that $|\mbb{C}(\fml{E})| = n$
~\footnote{In the worst case, the number of CXps can be exponential.},
a CXp-Forest defined on the variable set $\fml{F}=\{1, \dots, m\}$ is a linear combination of
$n$ binary DTs and represents the following function:
\begin{equation} \label{eq:ch}
\cf_{r}(\fml{S}; \fml{E}) = \frac{1}{n} \times \sum_{\fml{Y}_{i}\in\mbb{C}(\fml{E})} \tn{ITE}(\fml{S} \cap \fml{Y}_{i} \neq \emptyset, w_i, 0)
\end{equation}
where the weight $w_i\in\mbb{R}$ represents the quantity of AExs covered by $\fml{Y}_i$
~\footnote{In the implementation, we use the ratio of AExs rather than the actual number of AExs as the weights.}.
The output value of $\cf_{r}(\fml{S}; \fml{E})$ represents the total
number of AExs, and this value is distributed evenly across all $n$ CXps.
We can also consider an unweighted CXp-Forest by setting $w_i=1$ for each weight.
\end{definition}

\begin{example}\label{ex:run01b}
For the running example~\cref{ex:run01a},
we can construct a CXp-Forest, as shown in~\cref{fig:rf01}, as follows:
\begin{enumerate}[leftmargin=2em]
\item
For each CXp $\fml{Y}_i$ with size $k$, build a corresponding DT $T_i$ consisting of
$k$ non-leaf nodes and $k+1$ leaf nodes.
Each feature $j$ in $\fml{Y}_i$ is mapped to exactly one non-leaf node of $T_i$ labeled $j$.
\item
The edge connecting a non-leaf node to its right child indicates that $j\in\fml{S}$,
while the edge connecting it to its left child indicates that $j\not\in\fml{S}$.
\item
The right child of each non-leaf node is a leaf node labeled with 1,
while the left child is either another non-leaf node or a leaf node 0.
\end{enumerate}
For CXp $\fml{Y}_{1} = \{1,2\}$, it covers one AEx: $\{(1,0,2)\}$,
CXp $\fml{Y}_{2} = \{1,3\}$ covers four AExs: $\{(0,1,0), (0,1,1), (1,1,0), (1,1,1)\}$,
and CXp $\fml{Y}_{3} = \{2,3\}$ covers two AExs: $\{(2,0,0), (2,0,1)\}$.
That is why the weights of the three CXp-tree is $w_1=1$, $w_2=4$, and $w_3=2$, respectively.
\end{example}

\begin{figure}[t]
\begin{center}
\scalebox{.76}{\forestset{
  my edge/.style={
    edge path={
      \noexpand\path [draw, \forestoption{edge}] (!u.parent anchor) -- +(0,-5pt) -| (.child anchor)\forestoption{edge label};
    },
  },
  BDT/.style={
    for tree={
      l=1.5cm, s sep=1.15cm,
      if n children=0{}{circle},
      draw=midblue,
      text=midblue,
      edge={my edge},
      edge=thick,
    },
  },
}

\begin{forest}
  BDT
  [{$1$}, label={[yshift=-6.85ex]{{\tiny1}}}
    [{$2$}, label={[yshift=-6.85ex]{{\tiny2}}},
      edge label={node[midway,left,xshift=-1.5pt] {{\scriptsize$\not\in\fml{S}$}}}
      [\dghlight{\textbf{0}}, label={[yshift=-5.125ex]{{\tiny4}}},
        edge label={node[midway,left,xshift=-0.5pt]
          {{\scriptsize$\not\in\fml{S}$}}}, rectangle]
      [\dghlight{\textbf{1}}, label={[yshift=-5.125ex]{{\tiny5}}},
        edge={very thick,draw=darkred}, edge label={node[midway,right,xshift=-0.575pt] {{\scriptsize$\in\fml{S}$}}}, rectangle]
    ]
    [\dghlight{\textbf{1}}, label={[yshift=-5.125ex]{{\tiny3}}},
      edge={very thick,draw=darkred}, edge label={node[midway,right,xshift=0.5pt] {{\scriptsize$\in\fml{S}$}}},
      rectangle]
  ]
  \node [below=2.1cm of current bounding box.center] {$T_1$ built from $\fml{Y}_{1} = \{1,2\}$};
  \node [below=2.4cm of current bounding box.center] {$w_1 = \#\aex(\fml{Y}_{1}) = 1$};
\end{forest}

\begin{forest}
  BDT
  [{$1$}, label={[yshift=-6.85ex]{{\tiny6}}}
    [{$3$}, label={[yshift=-6.85ex]{{\tiny7}}},
      edge label={node[midway,left,xshift=-1.5pt] {{\scriptsize$\not\in\fml{S}$}}}
      [\dghlight{\textbf{0}}, label={[yshift=-5.125ex]{{\tiny9}}},
        edge label={node[midway,left,xshift=-0.5pt]
          {{\scriptsize$\not\in\fml{S}$}}}, rectangle]
      [\dghlight{\textbf{1}}, label={[yshift=-5.125ex]{{\tiny10}}},
        edge={very thick,draw=darkred}, edge label={node[midway,right,xshift=-0.575pt] {{\scriptsize$\in\fml{S}$}}}, rectangle]
    ]
    [\dghlight{\textbf{1}}, label={[yshift=-5.125ex]{{\tiny8}}},
      edge={very thick,draw=darkred}, edge label={node[midway,right,xshift=0.5pt] {{\scriptsize$\in\fml{S}$}}},
      rectangle]
  ]
  \node [below=2.1cm of current bounding box.center] {$T_2$ built from $\fml{Y}_{2} = \{1,3\}$};
  \node [below=2.4cm of current bounding box.center] {$w_2 = \#\aex(\fml{Y}_{2}) = 4$};
\end{forest}

\begin{forest}
  BDT
  [{$2$}, label={[yshift=-6.85ex]{{\tiny11}}}
    [{$3$}, label={[yshift=-6.85ex]{{\tiny12}}},
      edge label={node[midway,left,xshift=-1.5pt] {{\scriptsize$\not\in\fml{S}$}}}
      [\dghlight{\textbf{0}}, label={[yshift=-5.125ex]{{\tiny14}}},
        edge label={node[midway,left,xshift=-0.5pt]
          {{\scriptsize$\not\in\fml{S}$}}}, rectangle]
      [\dghlight{\textbf{1}}, label={[yshift=-5.125ex]{{\tiny15}}},
        edge={very thick,draw=darkred}, edge label={node[midway,right,xshift=-0.575pt] {{\scriptsize$\in\fml{S}$}}}, rectangle]
    ]
    [\dghlight{\textbf{1}}, label={[yshift=-5.125ex]{{\tiny13}}},
      edge={very thick,draw=darkred}, edge label={node[midway,right,xshift=0.5pt] {{\scriptsize$\in\fml{S}$}}},
      rectangle]
  ]
  \node [below=2.1cm of current bounding box.center] {$T_3$ built from $\fml{Y}_{3} = \{2,3\}$};
  \node [below=2.4cm of current bounding box.center] {$w_3 = \#\aex(\fml{Y}_{3}) = 2$};
\end{forest}}
\end{center}
\label{fig:cxpf01}
\caption{CXp-Forest for $\fml{E}=(\kappa, ((2,1,2),1))$.}
\label{fig:rf01}
\end{figure}
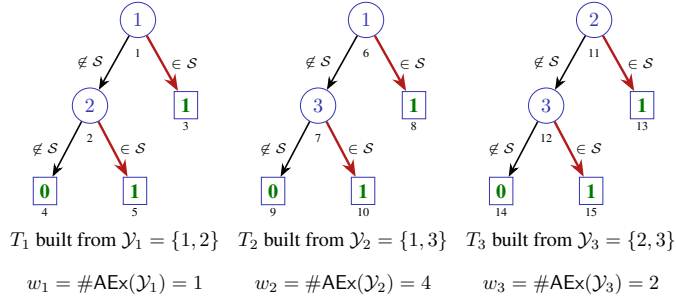

The weight $w_i$ of each CXp-tree comes from the quantity of AExs covered by that CXp.
It is important to note that an AEx cannot contribute
to the weight of more than one CXp-tree.
For example, suppose $\fml{Y}_{1} = \{1,2,3,4\}$ and $\fml{Y}_{2} = \{3,4,5\}$
are two CXps, and there is an AEx covered by these two CXps.
Then such AEx can be represented as
$(x_1 = v_1, x_2 = v_2, \star, \star, x_5 = v_5, \dots, x_m = v_m)$,
where $\star$ denotes any possible value in the domain $\mbb{D}_i$.
This implies that such AEx are covered by the weak CXp $\{3, 4\}$,
which is a contradiction.
Therefore, the set of AExs covered by two distinct CXps are disjoint.


\paragraph{Complexity of building a CXp-Forest.}
We analyze the complexity of building a CXp-Forest.
Given an oracle $\fml{O}_{\mbb{C}}(\fml{E})$ for computing $\mbb{C}(\fml{E})$,
and another oracle $\fml{O}_{\#}(\fml{E},\fml{Y})$ for measuring the quantity
of AExs covered by an CXp $\fml{Y}$, we can prove the following results.

\begin{restatable}{proposition}{PropComCXpF2} \label{prop:cxpf02}
The CXp-Forest can be constructed in polynomial time
given access to an $\fml{O}_{\mbb{C}}(\fml{E})$ and an $\fml{O}_{\#}(\fml{E},\fml{Y})$,
with a polynomial number of calls to an $\fml{O}_{\#}(\fml{E},\fml{Y})$.
\end{restatable}

\begin{restatable}{proposition}{PropComCXpDt} \label{prop:dt}
For decision trees, a CXp-Forest can be constructed
in polynomial time with respect to the number of paths in the given decision trees.
\end{restatable}

\section{\xhsc (\underline{A}dversarial e\underline{x}amples-based \underline{F}eature \underline{i}mportance) Scores} \label{sec:score}
This section presents two novel rigorous feature importance scores called \emph{\xhsc scores}, denoted as $\fs$.
These scores are based on the Shapley value and Banzhaf index.
To compute these feature importance scores, we pick the previously defined characteristic function~\cref{eq:ch}.
Then, the novel feature importance scores are formally defined as follows:
\begin{definition}[Shapley-like \xhsc scores]
\begin{equation} \label{eq:defo_s}
\fsn{S}(i) := \svn{S}(i;\fml{E}_r, \cf_{r})
\end{equation}
\end{definition}

\begin{definition}[Banzhaf-like \xhsc scores]
\begin{equation} \label{eq:defo_b}
\fsn{B}(i) := \svn{B}(i;\fml{E}_r, \cf_{r})
\end{equation}
\end{definition}
$\fs(i) > \fs(j)$ indicates that feature $i$ is more important than feature $j$, taking into account:
1) the quantity of additional AExs excluded by adding feature $i$ to the coalition $\fml{S}$, and  
2) the size of the coalition $\fml{S}$.

\subsection{Novel Properties for Feature Importance Scores} \label{ssec:npp}
Game-theoretic feature attribution methods satisfy a set of favourable properties.
In this paper, we consider the following properties:
\emph{Efficiency}, \emph{Symmetry}, \emph{Additivity}, \emph{Null player},
\emph{AXp-minimal monotonicity}, \emph{$\gamma$-efficiency}.
(See \cref{app:pfis} for the formal definitions of these properties.)
Next, we define new properties that feature importance scores should respect.

\paragraph{Independence from the model's output.}
Given two models $\fml{M}=(\fml{F},\mbb{F},\mbb{T},\tau)$, 
$\fml{M}'=(\fml{F},\mbb{F},\mbb{T}',\tau')$, where the elements of
$\mbb{T}'$ are related with the elements of $\mbb{T}$ by a bijection
$\mu:\mbb{T}\to\mbb{T}'$, such that
$\forall(\mbf{x}\in\mbb{F}).(\tau'(\mbf{x})=\mu(\tau(\mbf{x})))$.
For $\fml{M}$ and $\fml{M}'$, the respective explanation problems are
$\fml{E}$ and $\fml{E}'$.
Let $\sv$ and $\sv'$ represent two scores defined, respectively, on
$\fml{E}$ and $\fml{E}'$.
Then, $\sv$ respects independence from model's output if
$\forall(i\in\fml{F}).(\sv(i;\fml{E})=\sv'(i;\fml{E}'))$.

\paragraph{CXp-minimal monotonicity.}
If the set of CXps for feature $i$ is included in the set of
CXps for feature $j$, then the score for $i$ should be no greater than
the score for $j$,
\[
\forall(i,j\in\fml{F}).(\mbb{C}_i\subseteq\mbb{C}_j)\limply\sv(i;\fml{E},\cf)\le\sv(j;\fml{E},\cf).
\]

\paragraph{Consistency with relevancy.}
A score is consistent with relevancy if, for any explanation problem
$\fml{E}$, a feature takes a non-zero score iff it is relevant.


\subsection{Properties and Complexity} \label{ssec:pscore}

We investigate which of these desirable properties are satisfied by $\fsn{S}$ and $\fsn{B}$.
\begin{restatable}{proposition}{PropIrrRel} \label{prop:irrrel}
For any feature $j\in\fml{F}$,
if $j$ is irrelevant to $\fml{E}$ then $\fs(j) = 0$;
otherwise, $\fs(j) > 0$.
\end{restatable}

\begin{restatable}{proposition}{PropSh} \label{prop:pp_s}
$\fsn{S}$ satisfies the properties: Efficiency, Symmetry, Additivity, Null player, 
$\gamma$-Efficiency, 
CXp-minimal monotonicity, and Consistency with relevancy.
\end{restatable}

\begin{restatable}{proposition}{PropBz} \label{prop:pp_b}
$\fsn{B}$ satisfies the properties: Symmetry, Additivity, Null player, 
$\gamma$-Efficiency, 
CXp-minimal monotonicity, and Consistency with relevancy.
\end{restatable}

Furthermore, the exact computation of $\fsn{S}$
and $\fsn{B}$ for each relevant feature, 
and the computational complexity of $\fsn{S}$ and $\fsn{B}$
can be proved.
\begin{restatable}{proposition}{PropComS} \label{prop:comp_s}
For an arbitrary feature $j\in\mathfrak{F}(\fml{E})$,
$\fsn{S}(j) = \frac{1}{n} \sum_{\fml{Y}_{i}\in\mbb{C}(\fml{E}), j\in\fml{Y}_{i}} \frac{w_i}{|\fml{Y}_{i}|}$.
\end{restatable}
\begin{proof}
Suppose $|\mbb{C}(\fml{E})| = n$.
Due to the property of \emph{additivity}, we can compute the local $\fsn{S}(j)$ for each CXp-tree
and then sum them up to obtain the final $\fsn{S}(j)$.
For each CXp-tree built from a CXp $\fml{Y}_{i}$, since the function $w_i \tau_{i}$ is symmetric,
every feature in this CXp-tree will have the same score.
Additionally, since $\cf_{r}(\emptyset) = 0$, $\fsn{S}$ satisfies $\frac{\sum_{i=1}^{n} w_i}{n}$-Efficiency,
while for each CXp-tree, the local $\fsn{S}$ satisfies $\frac{w_i}{n}$-Efficiency.
This means that for any $j \in \fml{Y}_{i}$, its local score is $\frac{w_i}{n |\fml{Y}_{i}|}$.
Thus, we can infer that
\begin{equation} \label{eq:com_s}
\fsn{S}(j) =
\frac{1}{n} \sum_{\fml{Y}_{i}\in\mbb{C}(\fml{E}), j\in\fml{Y}_{i}} \frac{w_i}{|\fml{Y}_{i}|}.
\end{equation}
\end{proof}

\begin{restatable}{proposition}{PropComB} \label{prop:comp_b}
For an arbitrary feature $j\in\mathfrak{F}(\fml{E})$,
$\fsn{B}(j) = \frac{1}{n} \sum_{\fml{Y}_{i}\in\mbb{C}(\fml{E}), j\in\fml{Y}_{i}} \frac{w_i}{2^{|\fml{Y}_{i}|-1}}$,
and $\gamma = \sum_{j \in \mathfrak{F}(\fml{E})} \frac{1}{n} \sum_{\fml{Y}_{i}\in\mbb{C}(\fml{E}), j\in\fml{Y}_{i}} \frac{w_i}{2^{|\fml{Y}_{i}|-1}}$.
\end{restatable}
\begin{proof}
Suppose $|\mbb{C}(\fml{E})| = n$.
Due to the property of \emph{additivity}, we can compute the local $\fsn{B}(j)$ for each CXp-tree
and then sum them up to obtain the final $\fsn{B}(j)$.
For each CXp-tree built from a CXp $\fml{Y}_{i}$, since the function $w_i \tau_{i}$ is symmetric,
every feature in this CXp-tree will have the same score.
Besides, for any $j \in \fml{Y}_{i}$, to compute the local $\fsn{B}(j)$, note that 
$\Delta_{j}(\fml{S}) = \frac{w_i}{n}$ iff $\fml{S} \cap \fml{Y}_{i} = \emptyset$, and $\Delta_{j}(\fml{S}) = 0$ otherwise.
Evidently, there are $2^{|\fml{F}| - |\fml{Y}_{i}|}$ such $\fml{S}$, each with the same
coefficient $\frac{1}{2^{|\fml{F}|-1}}$. Hence, we have
\begin{equation} \label{eq:com_b}
\fsn{B}(j)
= \frac{1}{n} \sum_{\fml{Y}_{i}\in\mbb{C}(\fml{E}), j\in\fml{Y}_{i}} \frac{w_i 2^{|\fml{F}|-|\fml{Y}_{i}|}}{2^{|\fml{F}|-1}}
= \frac{1}{n} \sum_{\fml{Y}_{i}\in\mbb{C}(\fml{E}), j\in\fml{Y}_{i}} \frac{w_i}{2^{|\fml{Y}_{i}|-1}}.
\end{equation}
From which we can derive that
\begin{equation}
\gamma =
\sum_{j \in \mathfrak{F}(\fml{E})} \frac{1}{n}
\sum_{\fml{Y}_{i}\in\mbb{C}(\fml{E}), j\in\fml{Y}_{i}} \frac{w_i}{2^{|\fml{Y}_{i}|-1}}.
\end{equation}
\end{proof}

\begin{example} \label{ex:run01c}
For the CXp-Forest in~\cref{fig:rf01},
consider computing $\fsn{S}$, then we have
$\fsn{S}(1) = \frac{5}{6}$, $\fsn{S}(2) = \frac{1}{2}$, and $\fsn{S}(3) = 1$.
Note that $\fsn{S} = \fsn{B}$ since all CXps have a size less than 3.
(See \cref{app:clac} for the detailed calculations.)
\end{example}

\begin{restatable}{corollary}{PropComplexityA} \label{cor:complexity1}
Given a CXp-Forest, $\fsn{S}$ and $\fsn{B}$
can be computed in polynomial time on the size of the forest.
\end{restatable}
\begin{proof}
Given the results of~\cref{prop:dt,prop:comp_s,prop:comp_b}.
\end{proof}
It is important to note that the resulting scores
are independent of the input data distribution, as long as the distribution is a product distribution.
In fact, one can explicitly design the CXp-Forest to cancel the effect of the distribution.

\subsection{\xhsc versus Other Scores} \label{ssec:othsc}

\paragraph{Rigorous feature importance measures versus SHAP.}
It has been shown that the theory behind SHAP scores
may assign low or zero contributions to features that are deemed relevant
from the perspective of logic-based abduction~\cite{hms-ijar24,izza-aaai24}.
A direct implication of this result is that
SHAP scores do not satisfy the property of \emph{Consistency with relevancy}.
Additionally, note that SHAP scores also do not satisfy other properties, 
such as AXp-minimal monotonicity and CXp-minimal monotonicity.

For tabular data and safety-critical applications, such as medical diagnosis,
this inconsistency with relevancy can be problematic.
This deficit can be addressed by replacing the deployment of SHAP scores
with rigorous feature importance measures, including the novel \xhsc scores.
However, a notable downside of these rigorous feature importance measures
is that the number of AXps and CXps can both be exponential.

\paragraph{\xhsc versus rigorous feature importance measures.}
For the existing rigorous feature importance measures: FFA, WFFA~\cite{ignatiev-sat24}, and Responsibility Index~\cite{izza-aaai24},
they satisfy the property of \emph{Consistency with relevancy}, as these scores are based on AXps.
However, complexity-wise, we can prove the following positive results
regarding the computation of \xhsc scores.
\begin{restatable}{proposition}{PropComplexityC} \label{prop:complexity3}
There exists decision trees such that 
computing \xhsc scores is more efficient compared to existing rigorous feature importance scores,
including FFA, WFFA, and Responsibility Index.
\end{restatable}

\paragraph{Unweighted \xhsc scores.}
When all the weights $w_i$ are set to one, we call the computed \xhsc as unweighted \xhsc scores.
We show that, there are connections between unweighted \xhsc scores
and CXp-based Deeghan-Packel~\cite{deegan-ijgt78} scores.
\begin{restatable}{proposition}{PropUSWffa}
Unweighted $\fsn{S}$ is equivalent to CXp-based Deeghan-Packel scores.
\end{restatable}


\section{Preliminary Experimental Results} \label{sec:res}
This section assesses the proposed method on decision trees (classification and regression)~\cite{breiman-bk84},
boosted trees~\cite{chen2016xgboost}, and logistic regression~\cite{hosmer2013applied},
using a range of widely used tabular classification and regression datasets.
Moreover, the proposed method was evaluated in the context of Just-in-Time (JIT) defect prediction
~\cite{kamei2012large,lin2021impact,pornprasit2021pyexplainer}.
Our evaluation involved computing proposed \xhsc scores.
As a baseline, we used the public distribution of the SHAP tool
~\footnote{Available from~\url{https://github.com/slundberg/shap}.}
to compute SHAP scores.
Additionally, we computed WFFA~\cite{ignatiev-sat24}
based on the publicly available source code
~\footnote{\url{https://github.com/ffattr/ffa}}.
We compared different scores from two aspects: 1) the runtime for computing these scores,
and 2) the ranking of feature importance imposed by different scores.
(See~\cref{app:res} for the experimental setup.)
The source code for the implementation is available at \url{https://github.com/XuanxiangHuang/AxFi}.

\paragraph{Comparison of runtime.}
\begin{table*}[ht]
\centering
\scalebox{0.98}{
\begin{tabular}{llccc}
\toprule
Machine Learning models & Dataset & WFFA & \xhsc & SHAP \\
\toprule
\multirow{3}{*}{Classification tree}
& Adult & 0.020 & 0.006 & 26.251 \\
& COMPAS & 0.013 & 0.004 & 4.875 \\
& Recidivism & 0.039 & 0.009 & 25.979 \\
\midrule
\multirow{2}{*}{Regression tree}
& Auto-MPG & 0.003 & 0.002 & 1.506 \\
& Boston House Prices & 0.004 & 0.002 & 17.383 \\
\midrule
\multirow{3}{*}{Boosted tree}
& Diabetes & 0.846 & 0.969 & 0.001 \\
& Vehicle & 432.669 & 492.592 & 0.001 \\
& Wine-Recognition & 0.457 & 0.473 & 0.001 \\
\midrule
\multirow{2}{*}{Logistic regression}
& Openstack & 0.049 & 0.262 & 0.017 \\
& Qt & 0.053 & 0.324 & 0.022 \\
\bottomrule
\end{tabular}
}
\caption{Average runtime of computing WFFA, \xhsc, and SHAP.}
\label{tab:runtime}
\end{table*}

\cref{tab:runtime} presents the runtime performance of computing \xhsc, WFFA, and SHAP.
We can observe that computing SHAP scores is faster than both \xhsc and WFFA scores.
~\footnote{
For decision trees, we use the Kernel SHAP method to compute SHAP scores 
because the Tree SHAP method does not support decision trees trained using Orange3. 
This is why computing SHAP scores for decision trees is slower than \xhsc and WFFA.
}
Moreover, for the dataset \emph{Vehicle}, the runtime of WFFA and \xhsc is much larger than
that of SHAP. Several factors contribute to this result:
1) The number of AXps/CXps is large.
2) Explaining tree ensembles is challenging due to their complexity.
For decision trees, computing \xhsc scores is faster than WFFA scores
because decision trees satisfy the property of polynomial-time model counting~\cite{darwiche-jair02}.
However, since model counting is generally intractable, which holds for both boosted trees and logistic regression models
~\cite{vandenbroeck-aaai21,hm-ijcai24}, computing \xhsc scores is slower than computing WFFA scores.
In the experiments, we estimated the weights of each CXp via sampling instead of performing exact counting.
Specifically, we drew 5000 samples from the space defined by a CXp
and then computed the ratio of AExs to the total data points in this space as the weight.

\paragraph{Comparison of rankings.}
\begin{figure*}[ht]
\centering
\includegraphics[width=\linewidth]{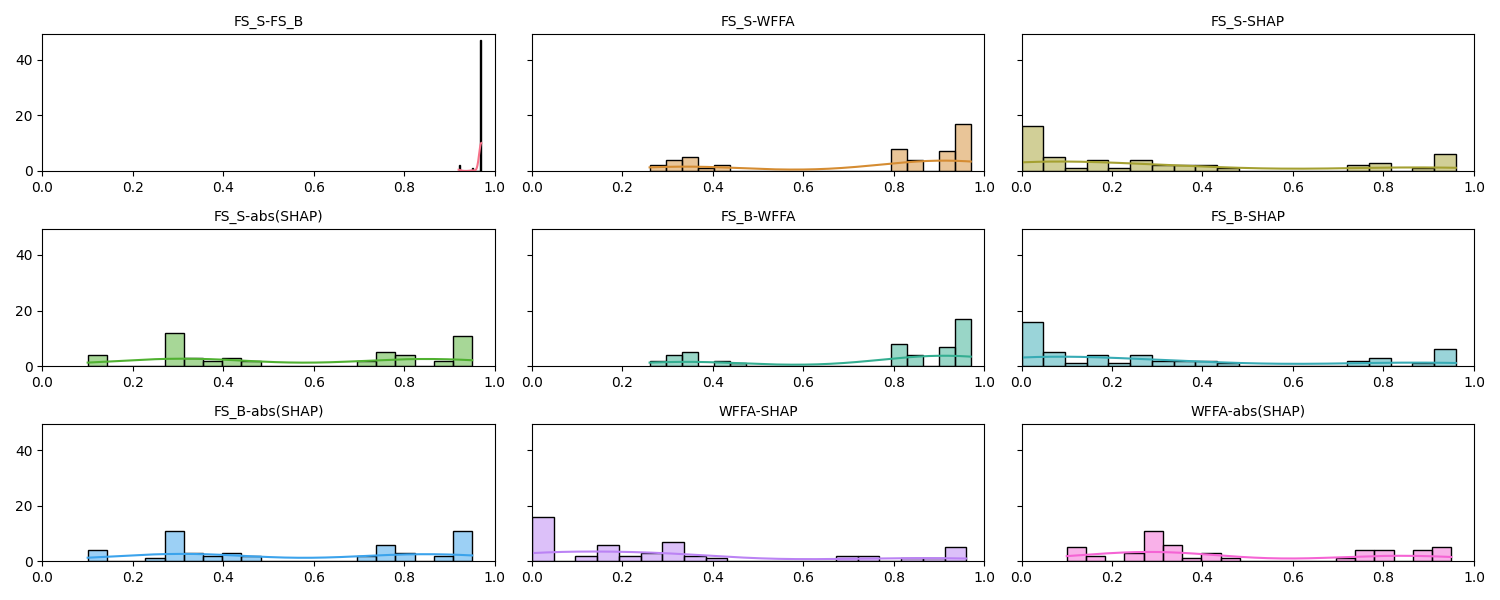}
\caption{Distribution of RBO values for all the tested instances of Recidivism.}
\label{fig:rbo_distri_recidivism0}
\end{figure*}

To compare the ranking of feature importance imposed by different scores, 
we computed the rank-biased overlap (RBO)~\cite{webber2010similarity} for each pair of scores.
RBO is a metric used to measure the similarity between two ranked lists, and is ranged between 0 and 1.
A higher RBO value indicates a greater degree of similarity between the two rankings.
For each dataset, we first computed \xhsc, WFFA, and SHAP scores for all tested samples.
Then extracted the order of feature importance from each score,
and this order was used to compute the RBO values.
For SHAP, we consider two orders: one based on the original SHAP scores and the other based on the absolute values of the SHAP scores.
Moreover, given the fact that most people are interested in top-ranked features,
we set \emph{persistence} to $0.5$ and \emph{depth} to $5$ in our setting, that is,
we put greater weights on the top-5 features.

The distributions of the resulting RBO values for all the test instances from the dataset Recidivism
is shown in~\cref{fig:rbo_distri_recidivism0}.
(See~\cref{app:res} for the distributions of the resulting RBO values for other datasets.)
In addition, nine subfigures depict the the distribution of RBO values
of comparing two different metrics.
For example, the subfigure in the upper-left corner shows the
the distribution of RBO values by comparing $\fsn{S}$ and $\fsn{B}$
(Note that \emph{abs(SHAP)} refers to SHAP scores in absolute values.)

The first observation is that RBO values for $\fsn{S}$ and $\fsn{B}$ often reach 1.0,
that is, top features identified by both metrics are very similar.
The second observation is that the two metrics, $\fs$ and WFFA, identified different top features,
and the distribution of RBO values for these two metrics is generally flat.
Even though the two metrics are based on formal explanations and there is a
minimal hitting set duality between AXps and CXps, the top features can still vary.
Moreover, a similar trend can be observed when comparing $\fs$ with SHAP, and comparing $\fs$ with abs(SHAP).

\paragraph{Comparison of the ranking for a specific instance.}
\begin{figure*}[!ht]
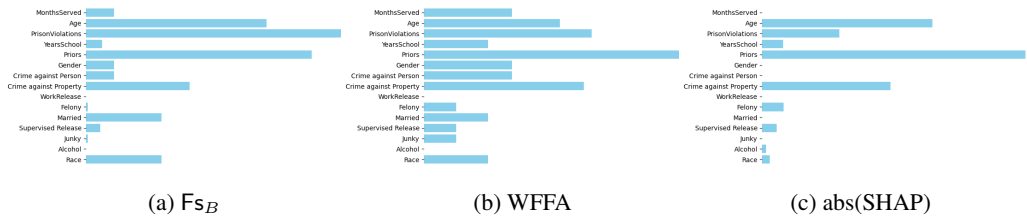

    \centering
    \begin{subfigure}[b]{0.345\textwidth}
        \centering
        \includegraphics[width=\linewidth]{figs/one_data_pt/F\_B\_recidivism\_46}
        \caption{$\fsn{B}$}
    \end{subfigure}
    \hspace{-0.5cm}
    \begin{subfigure}[b]{0.345\textwidth}
        \centering
        \includegraphics[width=\linewidth]{figs/one_data_pt/WFFA\_recidivism\_46}
        \caption{WFFA}
    \end{subfigure}
    \hspace{-0.5cm}
    \begin{subfigure}[b]{0.345\textwidth}
        \centering
        \includegraphics[width=\linewidth]{figs/one_data_pt/SHAP\_recidivism\_46}
        \caption{abs(SHAP)}
    \end{subfigure}
    \caption{
    Explanations for a sample from the Recidivism dataset, where the features are ordered as follows:
    Race, Alcohol, Junky, Supervised Release, Married, Felony, Work Release, Crime Against Property, 
    Crime Against Person, Gender, Priors, Years of School, Prison Violations, Age, Months Served.
    The considered sample is $(\mbf{v}, c) = ((0, 0, 0, 1, 1, 0, 0, 1, 0, 1, 0, 1, 2, 0, 0), 1)$.
    }   
    \label{fig:recidivism_inst}
\end{figure*}

\cref{fig:recidivism_inst} depicts the feature importance for a specific test instance
from Recidivism dataset, where Banzhaf-like \xhsc with WFFA and SHAP (in absolute value) were compared.
All three scores agree that the top four features
are \emph{Age}, \emph{Prison Violations}, \emph{Priors}, and \emph{Crime Against Property}, 
even though the order of these features is not the same.
For the feature \emph{Months Served}, its SHAP score is nearly zero, 
while its Banzhaf-like \xhsc score and WFFA score are not. 
Its WFFA score indicates that this feature is more important than \emph{Married}, 
but its Banzhaf-like \xhsc score indicates the reverse.

Although computing SHAP scores can be very efficient in practice, 
they may assign low or zero contribution to features that are deemed relevant from the perspective of logic-based abduction~\cite{hms-ijar24}.
For example, the features \emph{Crime Against Person} and \emph{Married} have almost zero SHAP scores, 
but their \xhsc and WFFA scores are not very low. 
In high-stakes and safety-critical scenarios, \xhsc can be a reliable measure that offers rigorous guarantees.
Moreover, compared to WFFA, \xhsc can offer a different and informative perspective on feature importance,
making it a valuable alternative and complement to existing approaches.



\section{Limits and Extensions} \label{sec:limit}

\paragraph{Feature dependency in CXp-Forest.}
While we assume feature independence in this paper, 
real-world data often contains dependent features.
A feasible solution is to encode feature dependencies as constraints, 
represent them using CXp-trees, 
and apply their outputs as penalties to reduce the value of the characteristic function.

\paragraph{Exponential number of CXps.}
A possible solution is to compute AXps/CXps with respect to
a subset of the feature space, $\mbb{S} \subseteq \mbb{F}$.
This subset can be formed by considering samples from the feature space
or all points within a small distance of the given instance $\mbf{v}$.


\section{Conclusions and Future Work} \label{sec:conc}
This paper proposes two novel rigorous feature importance scores, 
called \xhsc scores, for quantifying the effectiveness of excluding adversarial examples.
In contrast to earlier work on rigorous feature attribution, 
this paper also considers non-WAXp sets in the vicinity of the target data point.
This enables us to provide more detailed and informative insights regarding feature contribution.
\xhsc scores are computed via a novel characteristic function called contrastive explanation forests.
Moreover, the properties and computational complexity of these scores have been proven.
Compared to SHAP scores, \xhsc offers rigorous guarantees.
Moreover, in the case of decision trees, \xhsc can be computed more efficiently than existing rigorous feature importance measures.
Thus, in high-stakes and safety-critical scenarios, \xhsc serves as a valuable alternative to existing methods.
Future work will focus on addressing the limitations of the proposed methods.


\section*{Acknowledgments}
This work was conducted as part of the DesCartes program and
was supported by the National Research Foundation, Prime Minister’s
Office, Singapore, under the Campus for Research Excellence and
Technological Enterprise (CREATE) program.
This work was supported in part by the Spanish Government under
grant PID2023-152814OB-I00, and by ICREA starting funds.
This work was also supported in part by the AI Interdisciplinary
Institute ANITI, funded by the French program ``Investing for the
Future -- PIA3'' under Grant agreement no.\ ANR-19-PI3A-0004.

\newtoggle{mkbbl}

\settoggle{mkbbl}{false}

\bibliographystyle{plain}

\iftoggle{mkbbl}{
   \bibliography{refs,xtra}

\begin{thebibliography}{10}

\bibitem{magazzeni-facct22}
Emanuele Albini, Jason Long, Danial Dervovic, and Daniele Magazzeni.
\newblock Counterfactual {S}hapley additive explanations.
\newblock In {\em FACCT}, pages 1054--1070, 2022.

\bibitem{algaba2019handbook}
Encarnaci{\'o}n Algaba, Vito Fragnelli, and Joaqu{\'\i}n S{\'a}nchez-Soriano.
\newblock {\em Handbook of the Shapley value}.
\newblock CRC Press, 2019.

\bibitem{barcelo-jmlr23}
Marcelo Arenas, Pablo Barcel{\'{o}}, Leopoldo~E. Bertossi, and Mika{\"{e}}l
  Monet.
\newblock On the complexity of {SHAP}-score-based explanations: Tractability
  via knowledge compilation and non-approximability results.
\newblock {\em J. Mach. Learn. Res.}, 24:63:1--63:58, 2023.

\bibitem{asuncion2007uci}
Arthur Asuncion, David Newman, et~al.
\newblock Uci machine learning repository.
\newblock \url{http://archive.ics.uci.edu/ml}, 2007.

\bibitem{banzhaf-rlr65}
John~F Banzhaf~III.
\newblock Weighted voting doesn't work: A mathematical analysis.
\newblock {\em Rutgers L. Rev.}, 19:317, 1965.

\bibitem{barcelo2020model}
Pablo Barcel{\'o}, Mika{\"e}l Monet, Jorge P{\'e}rez, and Bernardo
  Subercaseaux.
\newblock Model interpretability through the lens of computational complexity.
\newblock {\em Advances in neural information processing systems},
  33:15487--15498, 2020.

\bibitem{izza-aaai24}
Gagan Biradar, Yacine Izza, Elita Lobo, Vignesh Viswanathan, and Yair Zick.
\newblock Axiomatic aggregations of abductive explanations.
\newblock In {\em AAAI}, pages 11096--11104, 2024.

\bibitem{breiman-bk84}
L.~Breiman, J.~Friedman, R.A. Olshen, and C.J. Stone.
\newblock {\em Classification and regression trees}.
\newblock Chapman and Hall, 1984.

\bibitem{johnson-sttt23}
Christopher Brix, Mark~Niklas M{\"{u}}ller, Stanley Bak, Taylor~T. Johnson, and
  Changliu Liu.
\newblock First three years of the international verification of neural
  networks competition {(VNN-COMP)}.
\newblock {\em Int. J. Softw. Tools Technol. Transf.}, 25(3):329--339, 2023.

\bibitem{jordan-iclr19}
Jianbo Chen, Le~Song, Martin~J. Wainwright, and Michael~I. Jordan.
\newblock {L}-{S}hapley and {C}-{S}hapley: Efficient model interpretation for
  structured data.
\newblock In {\em ICLR}, 2019.

\bibitem{chen2016xgboost}
Tianqi Chen and Carlos Guestrin.
\newblock Xgboost: A scalable tree boosting system.
\newblock In {\em KDD}, pages 785--794, 2016.

\bibitem{lundberg-nips20}
Ian Covert, Scott~M. Lundberg, and Su{-}In Lee.
\newblock Understanding global feature contributions with additive importance
  measures.
\newblock In {\em NeurIPS}, 2020.

\bibitem{darwiche-lics23}
Adnan Darwiche.
\newblock Logic for explainable {AI}.
\newblock In {\em {LICS}}, pages 1--11, 2023.

\bibitem{darwiche-jair02}
Adnan Darwiche and Pierre Marquis.
\newblock A knowledge compilation map.
\newblock {\em J. Artif. Intell. Res.}, 17:229--264, 2002.

\bibitem{deegan-ijgt78}
John Deegan and Edward~W Packel.
\newblock A new index of power for simple $n$-person games.
\newblock {\em International Journal of Game Theory}, 7:113--123, 1978.

\bibitem{dubey-mor81}
Pradeep Dubey, Abraham Neyman, and Robert~James Weber.
\newblock Value theory without efficiency.
\newblock {\em Math. Oper. Res.}, 6(1):122--128, 1981.

\bibitem{machover-psr04}
Dan~S Felsenthal and Mosh{\'e} Machover.
\newblock A priori voting power: what is it all about?
\newblock {\em Political Studies Review}, 2(1):1--23, 2004.

\bibitem{frye2020asymmetric}
Christopher Frye, Colin Rowat, and Ilya Feige.
\newblock Asymmetric shapley values: incorporating causal knowledge into
  model-agnostic explainability.
\newblock In {\em NeurIPS}, pages 1229--1239, 2020.

\bibitem{nguyen-ieee-access21}
Daniel~Vidali Fryer, Inga Str{\"{u}}mke, and Hien~D. Nguyen.
\newblock Shapley values for feature selection: The good, the bad, and the
  axioms.
\newblock {\em {IEEE} Access}, 9:144352--144360, 2021.

\bibitem{szegedy-iclr15}
Ian~J. Goodfellow, Jonathon Shlens, and Christian Szegedy.
\newblock Explaining and harnessing adversarial examples.
\newblock In {\em ICLR}, 2015.

\bibitem{horn-bk12}
Roger~A Horn and Charles~R Johnson.
\newblock {\em Matrix analysis}.
\newblock Cambridge university press, 2012.

\bibitem{hosmer2013applied}
David~W Hosmer~Jr, Stanley Lemeshow, and Rodney~X Sturdivant.
\newblock {\em Applied logistic regression}.
\newblock John Wiley \& Sons, 2013.

\bibitem{himms-kr21}
Xuanxiang Huang, Yacine Izza, Alexey Ignatiev, and Jo{\~{a}}o Marques{-}Silva.
\newblock On efficiently explaining graph-based classifiers.
\newblock In {\em KR}, pages 356--367, 2021.

\bibitem{hms-ijar24}
Xuanxiang Huang and Joao Marques-Silva.
\newblock On the failings of {S}hapley values for explainability.
\newblock {\em International Journal of Approximate Reasoning}, page 109112,
  2024.

\bibitem{hm-ijcai24}
Xuanxiang Huang and Joao Marques-Silva.
\newblock Updates on the complexity of shap scores.
\newblock In {\em IJCAI-24}, pages 403--412, 2024.

\bibitem{inms-nips19}
Alexey Ignatiev, Nina Narodytska, and Joao Marques{-}Silva.
\newblock On relating explanations and adversarial examples.
\newblock In {\em NeurIPS}, pages 15857--15867, 2019.

\bibitem{blockbaum-aistats20}
Dominik Janzing, Lenon Minorics, and Patrick Bl{\"{o}}baum.
\newblock Feature relevance quantification in explainable {AI:} {A} causal
  problem.
\newblock In {\em AISTATS}, pages 2907--2916, 2020.

\bibitem{kamei2012large}
Yasutaka Kamei, Emad Shihab, Bram Adams, Ahmed~E Hassan, Audris Mockus, Anand
  Sinha, and Naoyasu Ubayashi.
\newblock A large-scale empirical study of just-in-time quality assurance.
\newblock {\em IEEE Transactions on Software Engineering}, 39(6):757--773,
  2012.

\bibitem{kulynych2017feature}
Bogdan Kulynych and Carmela Troncoso.
\newblock Feature importance scores and lossless feature pruning using banzhaf
  power indices.
\newblock {\em arXiv preprint arXiv:1711.04992}, 2017.

\bibitem{friedler-icml20}
I.~Elizabeth Kumar, Suresh Venkatasubramanian, Carlos Scheidegger, and
  Sorelle~A. Friedler.
\newblock Problems with {S}hapley-value-based explanations as feature
  importance measures.
\newblock In {\em ICML}, pages 5491--5500, 2020.

\bibitem{friedler-nips21}
Indra Kumar, Carlos Scheidegger, Suresh Venkatasubramanian, and Sorelle~A.
  Friedler.
\newblock Shapley residuals: Quantifying the limits of the {S}hapley value for
  explanations.
\newblock In {\em NeurIPS}, pages 26598--26608, 2021.

\bibitem{lhms-aaai25}
Olivier L{\'e}toff{\'e}, Xuanxiang Huang, and Joao Marques-Silva.
\newblock Towards trustable shap scores.
\newblock In {\em Proceedings of the AAAI Conference on Artificial
  Intelligence}, volume~39, pages 18198--18208, 2025.

\bibitem{lin2021impact}
Dayi Lin, Chakkrit Tantithamthavorn, and Ahmed~E Hassan.
\newblock The impact of data merging on the interpretation of cross-project
  just-in-time defect models.
\newblock {\em IEEE Transactions on Software Engineering}, 48(8):2969--2986,
  2021.

\bibitem{lundberg-nips17}
Scott~M. Lundberg and Su{-}In Lee.
\newblock A unified approach to interpreting model predictions.
\newblock In {\em NeurIPS}, pages 4765--4774, 2017.

\bibitem{msh-cacm24}
Joao Marques-Silva and Xuanxiang Huang.
\newblock Explainability is not a game.
\newblock {\em Commun. ACM}, 67(7):66–75, jul 2024.

\bibitem{msi-aaai22}
Joao Marques{-}Silva and Alexey Ignatiev.
\newblock Delivering trustworthy {AI} through formal {XAI}.
\newblock In {\em AAAI}, pages 12342--12350, 2022.

\bibitem{mcintosh2018fix}
Shane McIntosh and Yasutaka Kamei.
\newblock Are fix-inducing changes a moving target? a longitudinal case study
  of just-in-time defect prediction.
\newblock In {\em ICSE}, pages 560--560, 2018.

\bibitem{miller2019explanation}
Tim Miller.
\newblock Explanation in artificial intelligence: Insights from the social
  sciences.
\newblock {\em Artificial intelligence}, 267:1--38, 2019.

\bibitem{molnar-bk20}
Christoph Molnar.
\newblock Interpretable machine learning.
\newblock \url{https://christophm.github.io/interpretable-ml-book/}, 2020.
\newblock Last accessed: 2024-06-11.

\bibitem{novelli2023accountability}
Claudio Novelli, Mariarosaria Taddeo, and Luciano Floridi.
\newblock Accountability in artificial intelligence: what it is and how it
  works.
\newblock {\em AI \& SOCIETY}, pages 1--12, 2023.

\bibitem{olson2017pmlb}
Randal~S Olson, William La~Cava, Patryk Orzechowski, Ryan~J Urbanowicz, and
  Jason~H Moore.
\newblock Pmlb: a large benchmark suite for machine learning evaluation and
  comparison.
\newblock {\em BioData mining}, 10:1--13, 2017.

\bibitem{patel2021high}
Neel Patel, Martin Strobel, and Yair Zick.
\newblock High dimensional model explanations: An axiomatic approach.
\newblock In {\em FAccT}, pages 401--411, 2021.

\bibitem{patel2021game}
Roma Patel, Marta Garnelo, Ian Gemp, Chris Dyer, and Yoram Bachrach.
\newblock Game-theoretic vocabulary selection via the shapley value and banzhaf
  index.
\newblock In {\em NAACL}, pages 2789--2798, 2021.

\bibitem{pornprasit2021pyexplainer}
Chanathip Pornprasit, Chakkrit Tantithamthavorn, Jirayus Jiarpakdee, Michael
  Fu, and Patanamon Thongtanunam.
\newblock Pyexplainer: Explaining the predictions of just-in-time defect
  models.
\newblock In {\em ASE}, pages 407--418. IEEE, 2021.

\bibitem{guestrin-kdd16}
Marco~T{\'{u}}lio Ribeiro, Sameer Singh, and Carlos Guestrin.
\newblock "why should {I} trust you?": Explaining the predictions of any
  classifier.
\newblock In {\em KDD}, pages 1135--1144, 2016.

\bibitem{robinson-bk03}
Derek~JS Robinson.
\newblock {\em An introduction to abstract algebra}.
\newblock Walter de Gruyter, 2003.

\bibitem{xai-bk19}
Wojciech Samek, Gr{\'{e}}goire Montavon, Andrea Vedaldi, Lars~Kai Hansen, and
  Klaus{-}Robert M{\"{u}}ller, editors.
\newblock {\em Explainable {AI:} Interpreting, Explaining and Visualizing Deep
  Learning}.
\newblock Springer, 2019.

\bibitem{shapley-ctg53}
Lloyd~S. Shapley.
\newblock A value for $n$-person games.
\newblock {\em Contributions to the Theory of Games}, 2(28):307--317, 1953.

\bibitem{shapley-apsr54}
Lloyd~S Shapley and Martin Shubik.
\newblock A method for evaluating the distribution of power in a committee
  system.
\newblock {\em American political science review}, 48(3):787--792, 1954.

\bibitem{kononenko-jmlr10}
Erik Strumbelj and Igor Kononenko.
\newblock An efficient explanation of individual classifications using game
  theory.
\newblock {\em J. Mach. Learn. Res.}, 11:1--18, 2010.

\bibitem{najmi-icml20}
Mukund Sundararajan and Amir Najmi.
\newblock The many {S}hapley values for model explanation.
\newblock In {\em ICML}, pages 9269--9278, 2020.

\bibitem{sundararajan2017axiomatic}
Mukund Sundararajan, Ankur Taly, and Qiqi Yan.
\newblock Axiomatic attribution for deep networks.
\newblock In {\em ICML}, pages 3319--3328. PMLR, 2017.

\bibitem{szegedy-iclr14}
Christian Szegedy, Wojciech Zaremba, Ilya Sutskever, Joan Bruna, Dumitru Erhan,
  Ian~J. Goodfellow, and Rob Fergus.
\newblock Intriguing properties of neural networks.
\newblock In {\em ICLR}, 2014.

\bibitem{vandenbroeck-aaai21}
Guy {Van den Broeck}, Anton Lykov, Maximilian Schleich, and Dan Suciu.
\newblock On the tractability of {SHAP} explanations.
\newblock In {\em AAAI}, pages 6505--6513, 2021.

\bibitem{watson-facct22}
David~S. Watson.
\newblock Rational {S}hapley values.
\newblock In {\em FAccT}, pages 1083--1094, 2022.

\bibitem{webber2010similarity}
William Webber, Alistair Moffat, and Justin Zobel.
\newblock A similarity measure for indefinite rankings.
\newblock {\em ACM Transactions on Information Systems (TOIS)}, 28(4):1--38,
  2010.

\bibitem{wegener2000branching}
Ingo Wegener.
\newblock {\em Branching programs and binary decision diagrams: theory and
  applications}.
\newblock SIAM, 2000.

\bibitem{barrett-nips23}
Min Wu, Haoze Wu, and Clark~W. Barrett.
\newblock {VeriX}: Towards verified explainability of deep neural networks.
\newblock In {\em NeurIPS}, 2023.

\bibitem{ignatiev-sat24}
Jinqiang Yu, Graham Farr, Alexey Ignatiev, and Peter~J. Stuckey.
\newblock {Anytime Approximate Formal Feature Attribution}.
\newblock In {\em SAT}, pages 30:1--30:23, 2024.

\bibitem{yu2019s}
Ronald Yu and Gabriele~Spina Al{\`\i}.
\newblock What's inside the black box? ai challenges for lawyers and
  researchers.
\newblock {\em Legal Information Management}, 19(1):2--13, 2019.

\end{thebibliography}
}{

}

\clearpage

\appendix

\section{Proposed Properties for Feature Importance Scores} \label{app:pfis}
These properties originate from two sources: 1) those originally proposed in cooperative games by L. Shapley~\cite{shapley-ctg53} (properties 1 through 4); 
2) those introduced by \cite{izza-aaai24} (properties 5 and 6).
(Readers are also referred to~\cite{algaba2019handbook,izza-aaai24} for further details.)
Let $\sv:\fml{F}\to\mbb{R}$ denote scores or values resulting from the application
of an arbitrary power index (e.g., Banzhaf) in the context of explainability.
\begin{enumerate}[leftmargin=1.5em]
\item \textbf{Efficiency}~
A cooperative game is efficient if,
\[
\sum\nolimits_{i\in\fml{F}}\sv(i;\fml{E},\cf)=\cf(\fml{F})-\cf(\emptyset).
\]
\item \textbf{Symmetry}~
Two features $i,j\in\fml{F}$ are symmetric if
$\cf(\fml{S}\cup\{i\})=\cf(\fml{S}\cup\{j\})$ for
$\fml{S}\subseteq\fml{F}\setminus\{i,j\}$.
A score respects the property of symmetry if, for any $\fml{E}$ and
for any symmetric features $i,j\in\fml{F}$,
$\sv(i;\fml{E},\cf)=\sv(j;\fml{E},\cf)$.
\item \textbf{Additivity}~
Let $\cfn{1}$ and $\cfn{2}$ represent two characteristic functions.
Furthermore, for any $\fml{S}\subseteq\fml{F}$, let
$\cfn{1{+}2}(\fml{S})=\cfn{1}(\fml{S})+\cfn{2}(\fml{S})$.
A score respects the property of additivity if for $i\in\fml{F}$, 
$\svn{1{+}2}(i;\fml{E},\cfn{1{+}2})=\svn{1}(i;\fml{E},\cfn{1})+\svn{2}(i;\fml{E},\cfn{2})$.
\item \textbf{Null player}~
For any features $i\in\fml{F}$, if $\cf(\fml{S})=\cf(\fml{S}\cup\{i\})$
for any $\fml{S}\subseteq\fml{F}$, then $i$ is a \emph{null player}
(or feature).
A score respects the null player property if for any null feature
$i\in\fml{F}$, $\sv(i)=0$.
\item \textbf{AXp-minimal monotonicity}~
If the set of AXps for feature $i$ is included in the set of
AXps for feature $j$, then the score for $i$ should be no greater than
the score for $j$,
\[
\forall(i,j\in\fml{F}).(\mbb{A}_i\subseteq\mbb{A}_j)\limply\sv(i;\fml{E},\cf)\le\sv(j;\fml{E},\cf).
\]
\item \textbf{$\gamma$-efficiency}~
Let $\gamma(;\fml{E})\in\mbb{R}$. A score is $\gamma$-efficient if,
\[
\sum\nolimits_{i\in\fml{F}}(\sv(i;\fml{E},\cf))=\gamma(;\fml{E}).
\]
$\gamma$-efficiency can be viewed as a mechanism of relaxing the
property of efficiency,
and has been the subject of past research~\cite{dubey-mor81}.
\end{enumerate}

\section{Calculations} \label{app:clac}
Calculation of $\cf_{r}(\fml{S})$ and $\Delta(\fml{S})$ for~\cref{ex:run01c}.
\begin{itemize}
\item
$\cf_{r}(\emptyset) = 0$.
\item
$\cf_{r}(\{1\}) = \frac{5}{3}$, $\cf_{r}(\{2\}) = \frac{3}{3}$, $\cf_{r}(\{3\}) = \frac{6}{3}$.
\item
$\cf_{r}(\{1,2\}) = \frac{7}{3}$, $\cf_{r}(\{1,3\}) = \frac{7}{3}$, $\cf_{r}(\{2,3\}) = \frac{7}{3}$.
\item
$\cf_{r}(\{1,2,3\}) = \frac{7}{3}$.
\end{itemize}
\begin{itemize}
\item
$\Delta_{1}(\emptyset) = \frac{5}{3}$, $\Delta_{1}(\{2\}) = \frac{4}{3}$, $\Delta_{1}(\{3\}) = \frac{1}{3}$, $\Delta_{1}(\{2,3\}) = 0$.
\item
$\Delta_{2}(\emptyset) = \frac{3}{3}$, $\Delta_{2}(\{1\}) = \frac{2}{3}$, $\Delta_{2}(\{3\}) = \frac{1}{3}$, $\Delta_{2}(\{1,3\}) = 0$.
\item
$\Delta_{3}(\emptyset) = \frac{6}{3}$, $\Delta_{3}(\{1\}) = \frac{2}{3}$, $\Delta_{3}(\{2\}) = \frac{4}{3}$, $\Delta_{3}(\{1,2\}) = 0$.
\end{itemize}

\section{Proofs} \label{app:proofs}

\PropComCXpDt*
\begin{proof}
We only consider DTs where 1) the tree paths represent a partition of the feature space $\mbb{F}$,
and 2) each point $\mbf{v}\in\mbb{F}$ is consistent with exactly one tree path.
Besides, for each non-leaf node labeled with feature $i$,
suppose its outgoing edge is labeled with a literal $x_i\in\mbb{E}_{i}$,
where $\mbb{E}_{i}\subseteq\mbb{D}_{i}$.

According to~\cite{himms-kr21}, the number of CXps for an explanation
problem define on a DT is polynomial on the number of tree paths.
Let us consider building a CXp-Forest.
Given an explanation problem $\fml{E} = (\fml{M}, (\mbf{v},q))$
where $\fml{M}$ is a DT, and given $\fml{Y}\in\mbb{C}(\fml{E})$.
To measure the quantity of AExs covered by this CXp $\fml{Y}$,
we proceed as follows:
\begin{enumerate}[leftmargin=1.5em]
\item
If a leaf node's label differs from the target label $q$ by more than 
$\delta$ (in classification problem $\delta=0$), change its label to 0.
Otherwise, change its label to 1.
For features in $\fml{F}\setminus\fml{Y}$, fix them to their given values $v_i$.
With these operations, non-leaf nodes having identical successors can be reduced
from $\fml{M}$ and some paths can be dropped from $\fml{M}$,
resulting in a new DT $\fml{M}'$ which is defined on $\fml{Y}$.
\item
Check each non-leaf nodes of $\fml{M}'$ to determine the new feature space $\mbb{F}'$.
If $\fml{M}'$ does not depend on $x_i$, let $\mbb{D}'_{i} = \mbb{D}_{i}$,
otherwise if $\fml{M}'$ depends on $x_i$, collect all literals $x_i\in\mbb{E}_{i}$
and let $\mbb{D}'_{i} = \bigcup_{i} \mbb{E}_{i}$.
\item
For each tree path of $\fml{M}'$, measure the quantity of points that output 0,
and sum them up.
\end{enumerate}
Clearly, each step can be done in polynomial time with respect to the size of the resulting DT $\fml{M}'$.
Hence, building a CXp-Forest
can be done in polynomial time with respect to the number of paths in the given DT.
\end{proof}

\PropIrrRel*
\begin{proof}
The characteristic function $\cf_{r}(\fml{S}; \fml{E})$ is defined on relevant features,
so all irrelevant features will have $\fs(j) = 0$.
For any relevant feature $j$,
it is clear that $\Delta_{j}(\fml{S}) \ge 0$ for any set $\fml{S}$,
and $\Delta_{j}(\emptyset) > 0$ always holds.
Therefore, $\fs(j) > 0$ for any feature $j$ that is relevant to $\fml{E}$.
\end{proof}

\PropSh*
\begin{proof}
We prove the properties one by one:
\begin{enumerate}
\item
$\fsn{S}$ satisfies \emph{Efficiency}, \emph{Symmetry}, \emph{Additivity}, and \emph{Null player},
which follows from the uniqueness of the Shapley value~\cite{shapley-ctg53}.
\item
Let $\fml{E}$ be an explanation problem such that
$\mbb{C}(\fml{E}) = \{\{1\}, \{2,3\}\}$ and $\mbb{A}(\fml{E}) = \{\{1,2\}, \{1,3\}\}$,
then we have $\mbb{A}_{2}\subseteq\mbb{A}_{1}$.
Let  the weight associated with $\{1\}$ be 1,
and the weight associated with $\{2,3\}$ be 5.
Then it is easy to verify that $\fsn{S}(1) = \frac{1}{2}$,
and $\fsn{S}(2) = \frac{5}{4}$,
which means $\fsn{S}(1) < \fsn{S}(2)$.
So $\fsn{S}$ does not satisfy \emph{AXp-minimal monotonicity}.
\item
Since $\fsn{S}$ satisfies \emph{Efficiency}, it also 
satisfies \emph{$\gamma$-Efficiency}, where $\gamma = 1$.
\item
The similarity predicate $\similar$ abstracts away whether the underlying model
implements classification or regression, and the bijection $\mu$ does not affect the similarity predicate.
If the weights of two CXp-Forests (one for $\fml{M}$ and one for $\fml{M}'$) are identical,
then $\fsn{S}$ satisfies \emph{Independence from the model's output}; otherwise, it does not.
\item
Given two features $i,j\in\fml{F}$ such that $\mbb{C}_i\subseteq\mbb{C}_j$,
it means if a CXp-tree contains feature $i$, then it must contain feature $j$.
Given a set $\fml{S}\subseteq\fml{F}$, if $i\not\in\fml{S}$ and $j\not\in\fml{S}$,
then $\Delta_{i}(\fml{S}) \le \Delta_{j}(\fml{S})$.
If $i\in\fml{S}$, then $\Delta_{j}(\fml{S}) \ge 0$.
If $j\in\fml{S}$, then $\Delta_{j}(\fml{S}) = 0$.
Hence, $\fsn{S}(i) \le \fsn{S}(j)$, that is,
$\fsn{S}$ satisfies \emph{CXp-minimal monotonicity}.
\item
$\fsn{S}$ satisfies \emph{Consistency with relevancy}, as implied by~\cref{prop:irrrel}.
\end{enumerate}
\end{proof}

\PropBz*
\begin{proof}
We prove the properties one by one:
\begin{enumerate}
\item
If $\fsn{B}$ satisfies \emph{Efficiency}, then that would falsify Shapley's theorem~\cite{shapley-ctg53},
so $\fsn{B}$ does not satisfy \emph{Efficiency}.
\item
Since the sums for two symmetric features $i$ and $j$ both use all sets (so the same sets),
$\Delta_{i} = \Delta_{j}$ (due to symmetry) and the same coefficient and so $\fsn{B}(i) = \fsn{B}(j)$,
i.e. $\fsn{B}$ satisfies \emph{Symmetry}.
\item
By construction, the sum operator and the $\Delta_{i}(\fml{S})$ are linear.
As the multiplicative constants do not depend on the characteristic function, the score is linear.
So $\fsn{B}$ satisfies \emph{Additivity}.
\item
The summations are trivially 0 by construction when $\Delta_{i}(\fml{S}) = 0$ for every set $\fml{S}\subseteq \fml{F}$.
$\fsn{B}$ satisfies \emph{Null player}.
\item
$\fsn{B}$ does not satisfy \emph{AXp-minimal monotonicity} for the same reason presented in~\cref{prop:pp_s}.
\item
$\fsn{B}$ satisfies \emph{$\gamma$-efficiency}, as implied by~\cref{prop:comp_b}, where
$\gamma = \sum_{j \in \mathfrak{F}(\fml{E})} \frac{1}{n} \sum_{\fml{Y}_{i}\in\mbb{C}(\fml{E}), j\in\fml{Y}_{i}} \frac{w_i}{2^{|\fml{Y}_{i}|-1}}$.
\item
The argument for whether $\fsn{B}$ satisfies \emph{Independence from the model's output} or not
is the same as presented in~\cref{prop:pp_s}.
\item
$\fsn{B}$ satisfies \emph{CXp-minimal monotonicity} for the same reason presented in~\cref{prop:pp_s}.
\item
$\fsn{B}$ satisfies \emph{Consistency with relevancy}, as implied by~\cref{prop:irrrel}.
\end{enumerate}
\end{proof}

\PropComplexityC*
\begin{proof}
The complexity of computing \xhsc scores has been shown in~\cref{cor:complexity1}.

For Responsibility Index, its definition is:
\[
\svn{Resp}(i;\fml{E}) := \max\left\{\left.\frac{1}{|\fml{S}|}\,\right|\,\fml{S}\in\mbb{A}_i(\fml{E})\right\}.
\]
It is evident that computing this score is at least as hard as computing a cardinality-minimal AXp.
Moreover, it is well-known that, for decision trees, finding one cardinality-minimal AXp is NP-hard~\cite{barcelo2020model}.
Hence, the problem of computing Responsibility Index is also NP-hard for decision trees.

The definitions of FFA and WFFA are, respectively, as follows:
\[
\svn{FFA}(i;\fml{E}) := \sum\nolimits_{\fml{S}\in\mbb{A}_i(\fml{E})} \left(\frac{1}{|\mbb{A}(\fml{E})|}\right),
\]
\[
\svn{WFFA}(i;\fml{E}) := \sum\nolimits_{\fml{S}\in\mbb{A}_i(\fml{E})} \left(\frac{1}{(|\fml{S}|\times|\mbb{A}(\fml{E})|)}\right).
\]
It is sufficient to construct a decision tree where the number of AXps is exponential.
Suppose we have a DT classifier $\fml{M}$ defined on $\fml{F} = \{1, \dots, m\}$, with $m=3k$, and $\fml{K}=\{0,1\}$.
Moreover, $\fml{F} = X \cup Y$, where $X = \{1, \dots, 2k\}$ and $Y = \{1, \dots, k\}$.
Let $\kappa$ be the classification function of $\fml{M}$.
$\fml{M}$ is composed of $k$ gadgets. The $i$-th gadget $G_i$ is defined on $\{x_{2i-1}, x_{2i}, y_{i}\}$,
and is described as follows:
\begin{equation}
\scriptstyle
\begin{array}{ll}
\\[1pt]
\tn{IF}~[(x_{2i-1} = 0) \land (x_{2i} = 1)]~\tn{THEN}~\kappa(\cdot)=1
\\[2.75pt] \nonumber
\tn{IF}~[(x_{2i-1} = 0) \land (x_{2i} = 0)]~\tn{THEN}~\kappa(\cdot)=0
\\[2.75pt] \nonumber
\tn{IF}~[(x_{2i-1} = 1) \land (x_{2i} = 0) \land (y_{i} = 1)]~\tn{THEN}~\kappa(\cdot)=1
\\[2.75pt] \nonumber
\tn{IF}~[(x_{2i-1} = 1) \land (x_{2i} = 0) \land (y_{i} = 0)]~\tn{THEN}~\kappa(\cdot)=0
\\[2.75pt] \nonumber
\tn{IF}~[(x_{2i-1} = 1) \land (x_{2i} = 1)]~\tn{THEN}~G_{i+1}
\\[1pt] \nonumber
\end{array}
\end{equation}
Let $\fml{E} = (\fml{M}, ((1,\dots,1), 1))$ be the instance to be explained.

We can compute the CXps as follows. Each gadget $i$ contributes two
CXps, namely $\{x_{2i-1},x_{2i}\}$ and $\{x_{2i},y_{i}\}$.
To compute one AXp, one must pick $x_{2i}$ or both $x_{2i-i}$ and $y_{i}$.
Since there are $k$ gadgets, we have $2^k$ AXps.
That is to say, for this explanation problem defined on a DT $\fml{M}$,
the number of CXps is polynomial on the size of $\fml{M}$,
but the number of AXps is exponential on the size of $\fml{M}$.
\end{proof}

\PropUSWffa*
\begin{proof}
Let $|\mbb{C}(\fml{E})| = n$ and $w_i = 1$ for every CXp-tree.
The unweighted $\fsn{S}$ for each feature $j$ is
\begin{equation}
\fsn{S}(j)
= \frac{1}{n} \sum_{\fml{Y}_{i}\in\mbb{C}(\fml{E}), j\in\fml{Y}_{i}} \frac{1}{|\fml{Y}_{i}|}.
\end{equation}
For the same feature $j\in\fml{F}$, its CXp-based Deeghan-Packel score is
\begin{equation}
\frac{\sum_{\fml{Y}_{i}\in\mbb{C}(\fml{E}), j\in\fml{Y}_{i}} |\fml{Y}_{i}|^{-1} }{|\mbb{C}(\fml{E})|}
= \frac{1}{n} \sum_{\fml{Y}_{i}\in\mbb{C}(\fml{E}), j\in\fml{Y}_{i}} \frac{1}{|\fml{Y}_{i}|}.
\end{equation}
Evidently, two scores are the same.
\end{proof}

\section{Experimental Details} \label{app:res}

\paragraph{Experimental environment.}
The experiments were performed on a MacBook Pro with a 6-Core Intel
Core i7 2.6 GHz processor with 16 GByte RAM, running macOS Sonoma.

\paragraph{Datasets and machine learning models.}
We selected commonly used tabular datasets in the fields of ML and XAI~\cite{asuncion2007uci,olson2017pmlb}
including six classification datasets and two regression datasets.
In addition, we selected the open-source Openstack and Qt datasets~\cite{mcintosh2018fix}, 
which are commonly used for JIT defect prediction.
Each dataset was randomly split into 80\% for training and 20\% for testing.
Decision trees were trained on three classification and two regression datasets using Orange3
~\footnote{\url{https://orangedatamining.com/}}.
Boosted trees were trained on three tabular datasets using XGBoost~\cite{chen2016xgboost},
with each boosted tree consisting of 20 trees of depth 3 per class.
Logistic regression models were trained on Openstack and Qt datasets using Scikit-learn
~\footnote{\url{https://scikit-learn.org/stable/}}.
Additionally, for two regression datasets,
we set $\delta=1.5$ when computing formal explanations.
From each dataset, 50 test instances were randomly selected for evaluation.
Moreover, a publicly available implementation\footnote{\url{https://github.com/changyaochen/rbo}.} of
RBO was used in our experiments.

\paragraph{Additional experimental results.}
The distribution of the resulting RBO values for all the test instances is shown in
~\cref{fig:rbo_distri_adult,%
fig:rbo_distri_compas,%
fig:rbo_distri_recidivism1,%
fig:rbo_distri_auto-mpg,%
fig:rbo_distri_boston_house_prices,%
fig:rbo_distri_diabetes,%
fig:rbo_distri_vehicle,%
fig:rbo_distri_wine-recognition,%
fig:rbo_distri_openstack,%
fig:rbo_distri_qt}.
Besides, \cref{tab:rbo_dt,tab:rbo_bt,tab:rbo_lr}
present a summary of the RBO values for all tested instances,
including the minimum, maximum, and mean RBO values for each pair of scores.
The second column indicates the pairs of scores being compared,
the value in the $i$-th row and $j$-th column represents the 
RBO value for the pair of scores in the $i$-th row, for the dataset in the $j$-th column.

\begin{figure*}[ht]
\centering
\includegraphics[width=\linewidth]{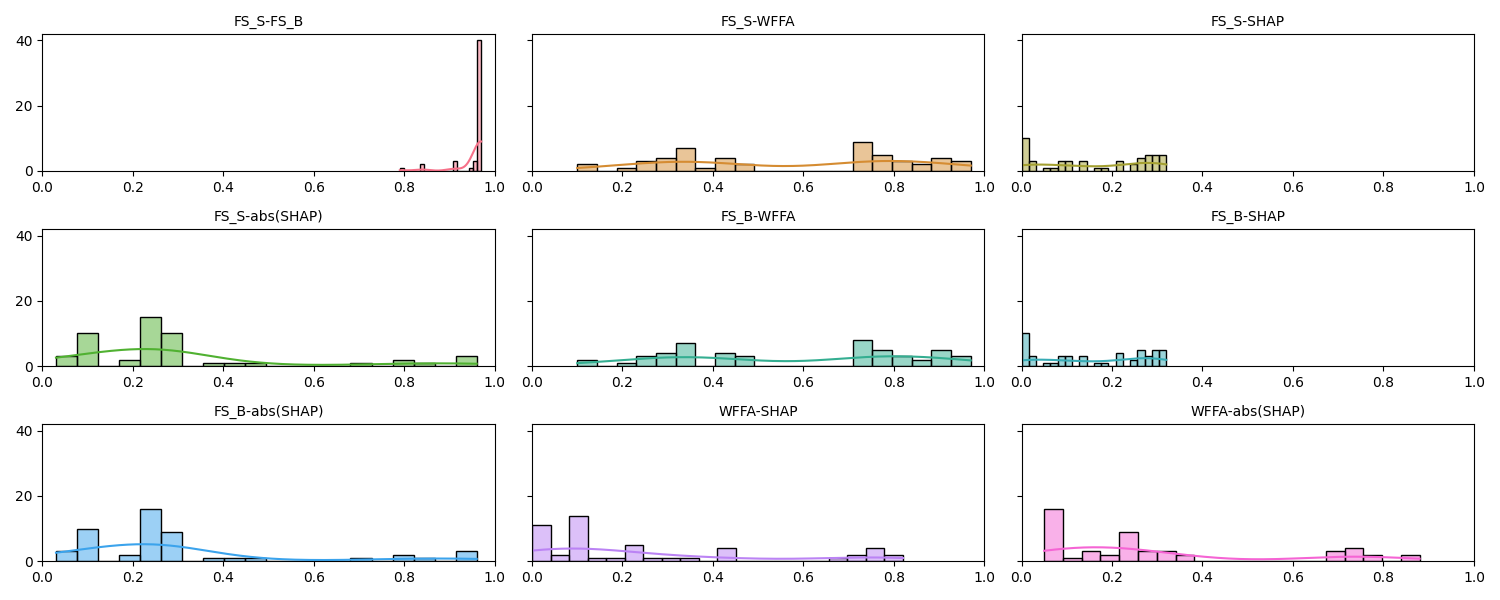}
\caption{Distribution of RBO values for all the tested instances of Adult.}
\label{fig:rbo_distri_adult}
\end{figure*}
\begin{figure*}[ht]
\centering
\includegraphics[width=\linewidth]{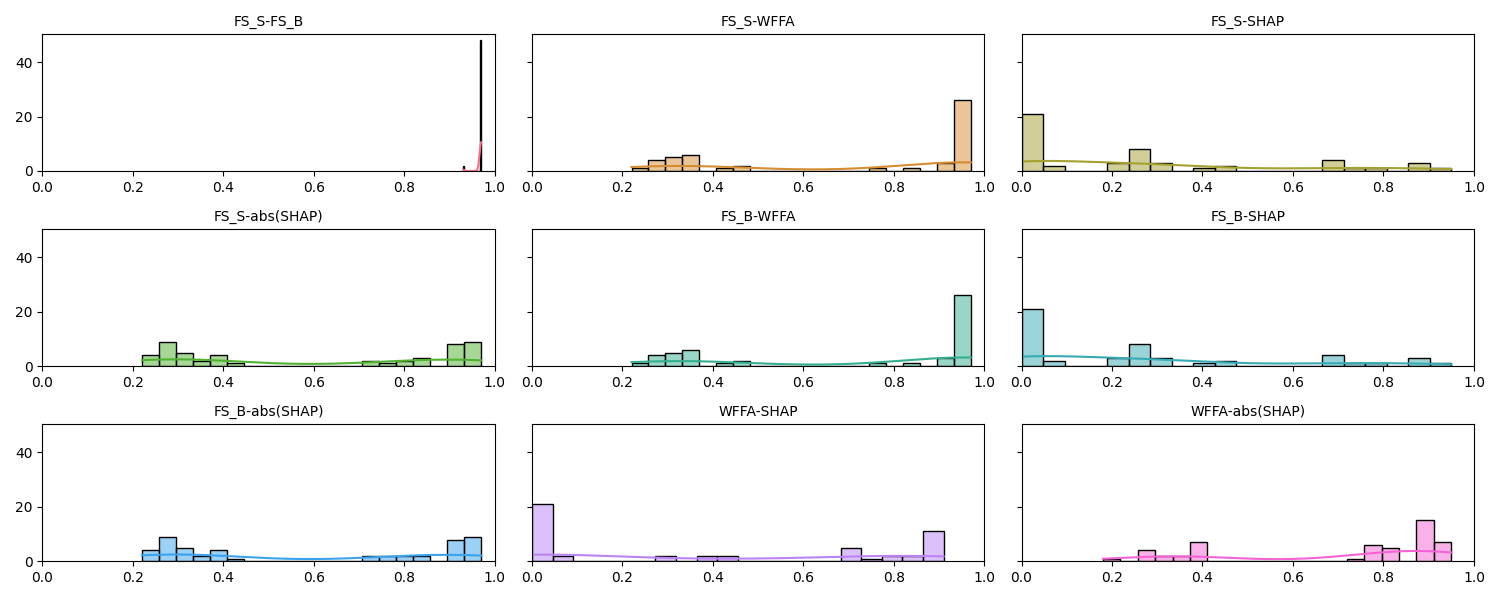}
\caption{Distribution of RBO values for all the tested instances of COMPAS.}
\label{fig:rbo_distri_compas}
\end{figure*}
\begin{figure*}[ht]
\centering
\includegraphics[width=\linewidth]{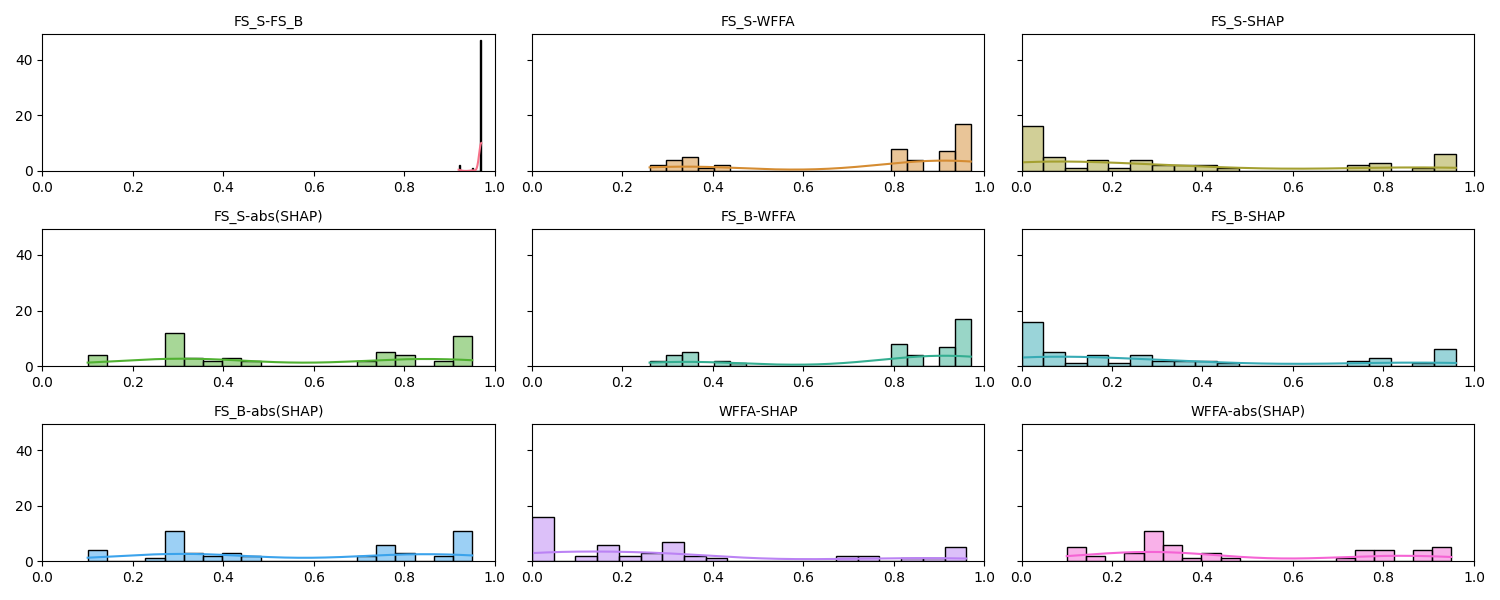}
\caption{Distribution of RBO values for all the tested instances of Recidivism.}
\label{fig:rbo_distri_recidivism1}
\end{figure*}
\begin{figure*}[ht]
\centering
\includegraphics[width=\linewidth]{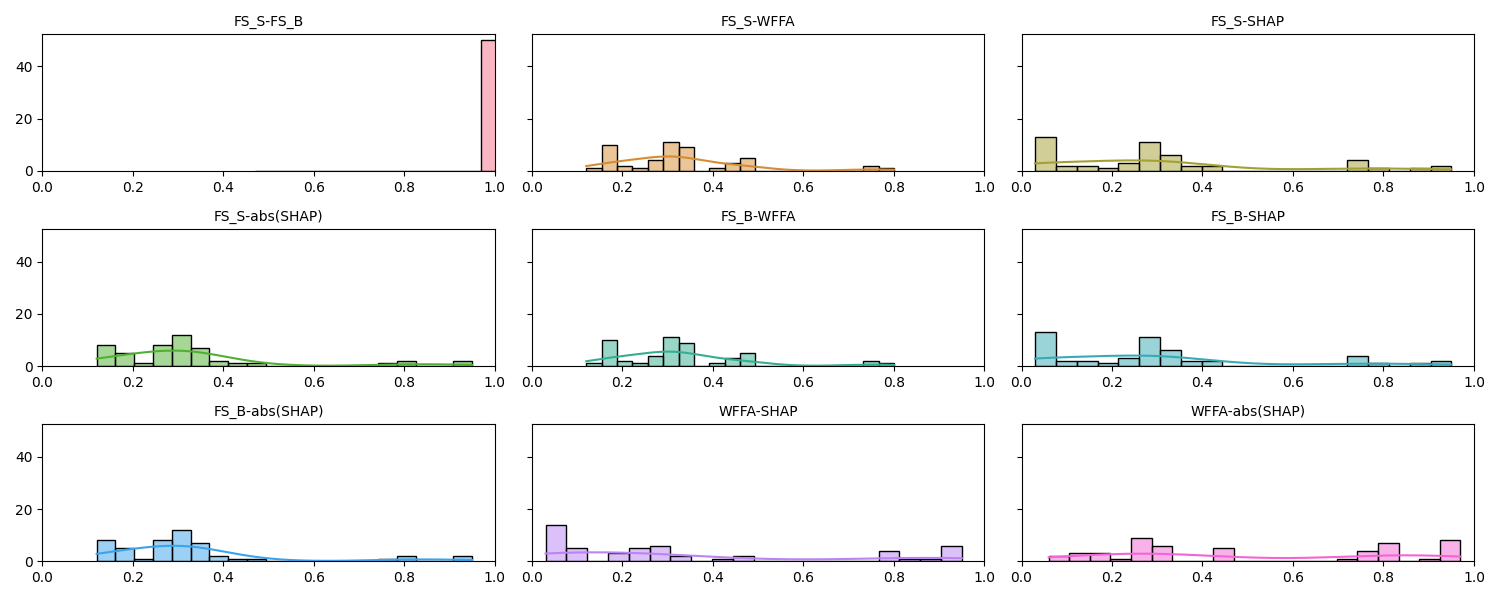}
\caption{Distribution of RBO values for all the tested instances of Auto-MPG.}
\label{fig:rbo_distri_auto-mpg}
\end{figure*}
\begin{figure*}[ht]
\centering
\includegraphics[width=\linewidth]{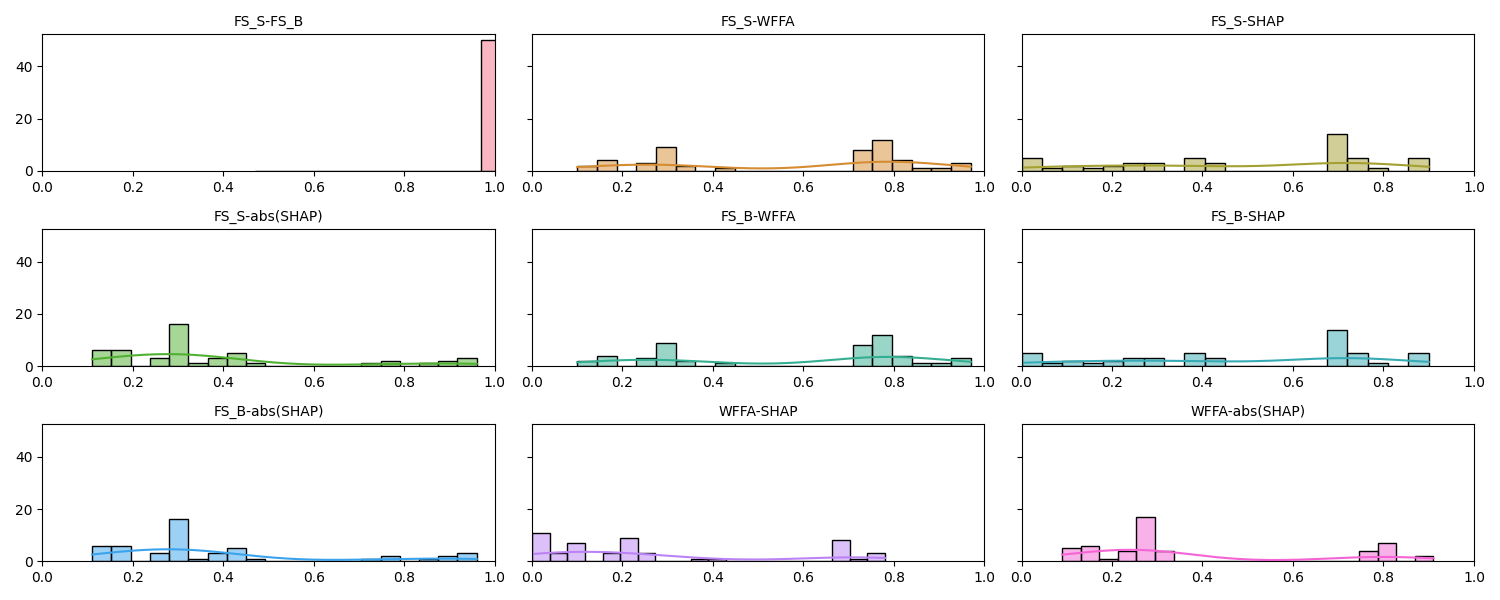}
\caption{Distribution of RBO values for all the tested instances of Boston House Prices.}
\label{fig:rbo_distri_boston_house_prices}
\end{figure*}
\begin{figure*}[ht]
\centering
\includegraphics[width=\linewidth]{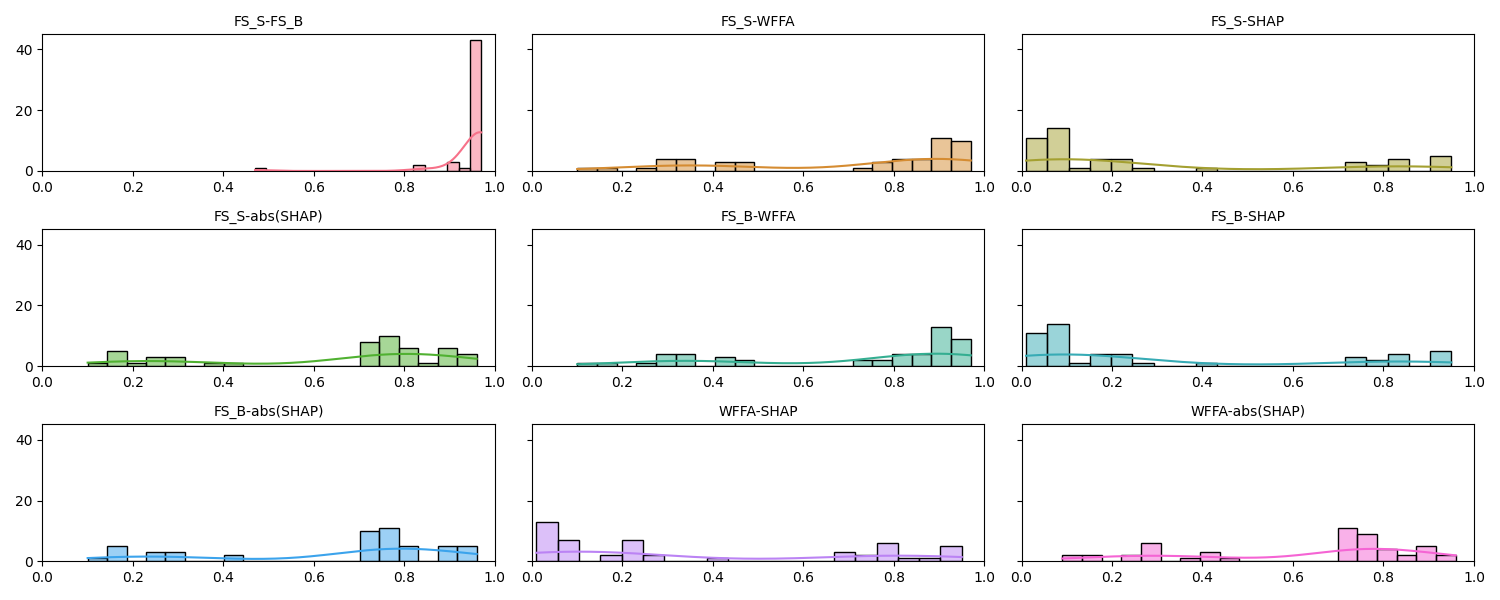}
\caption{Distribution of RBO values for all the tested instances of Diabetes.}
\label{fig:rbo_distri_diabetes}
\end{figure*}
\begin{figure*}[ht]
\centering
\includegraphics[width=\linewidth]{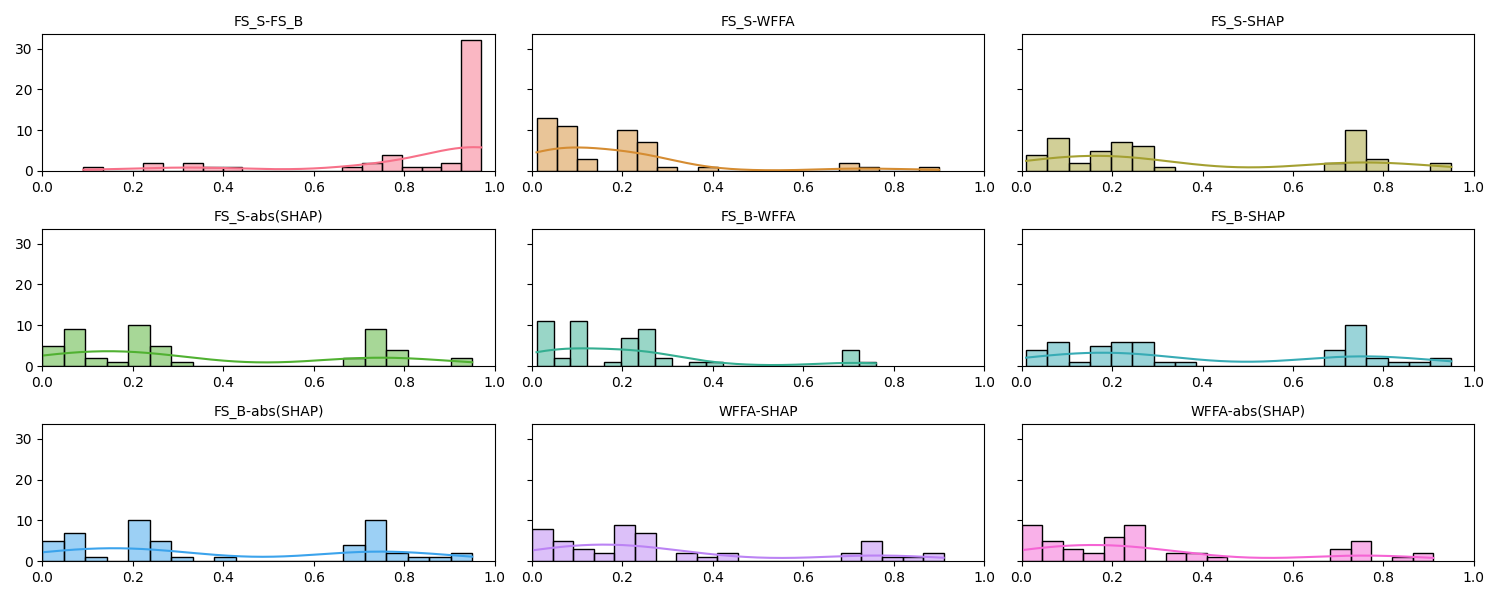}
\caption{Distribution of RBO values for all the tested instances of Vehicle.}
\label{fig:rbo_distri_vehicle}
\end{figure*}
\begin{figure*}[ht]
\centering
\includegraphics[width=\linewidth]{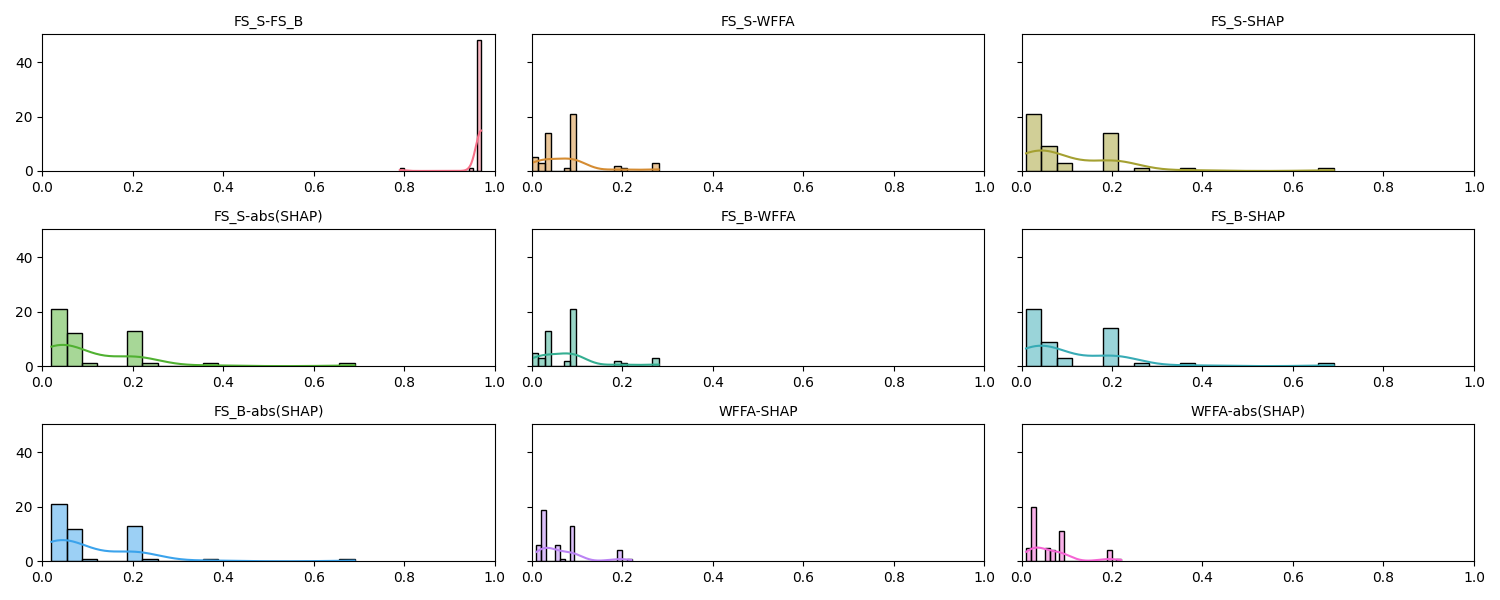}
\caption{Distribution of RBO values for all the tested instances of Wine-Recognition.}
\label{fig:rbo_distri_wine-recognition}
\end{figure*}
\begin{figure*}[ht]
\centering
\includegraphics[width=\linewidth]{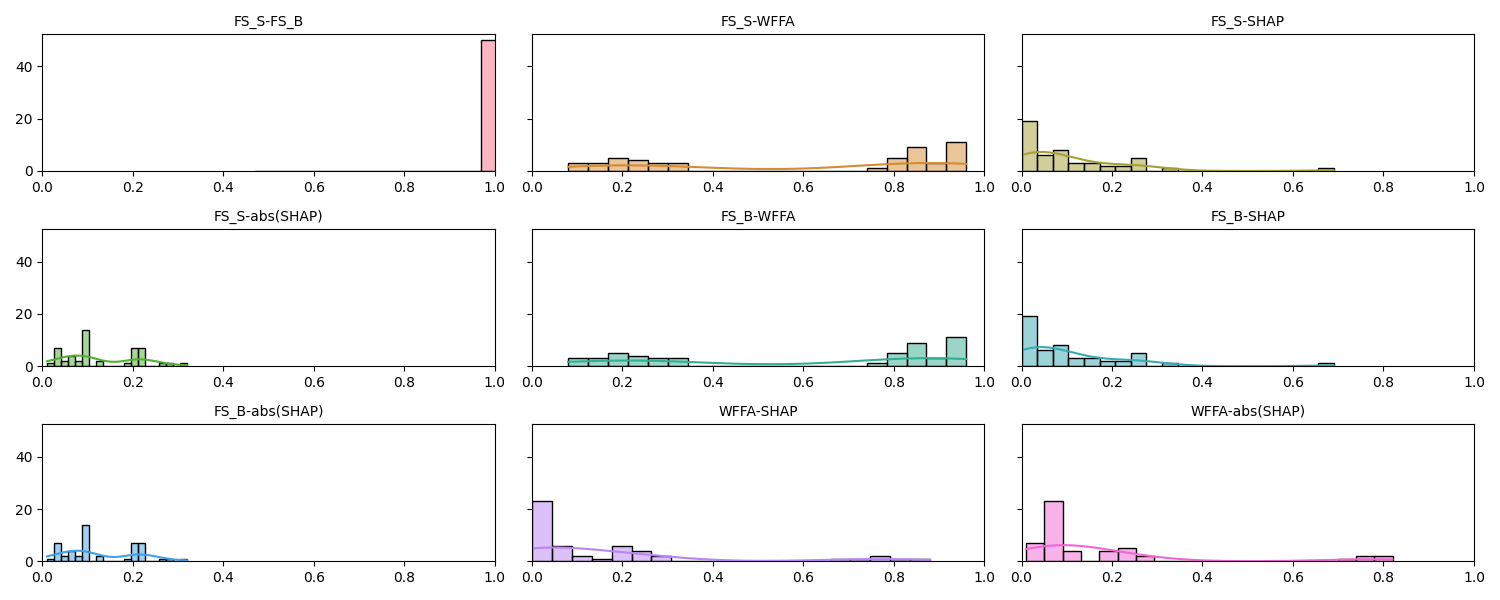}
\caption{Distribution of RBO values for all the tested instances of Openstack.}
\label{fig:rbo_distri_openstack}
\end{figure*}
\begin{figure*}[ht]
\centering
\includegraphics[width=\linewidth]{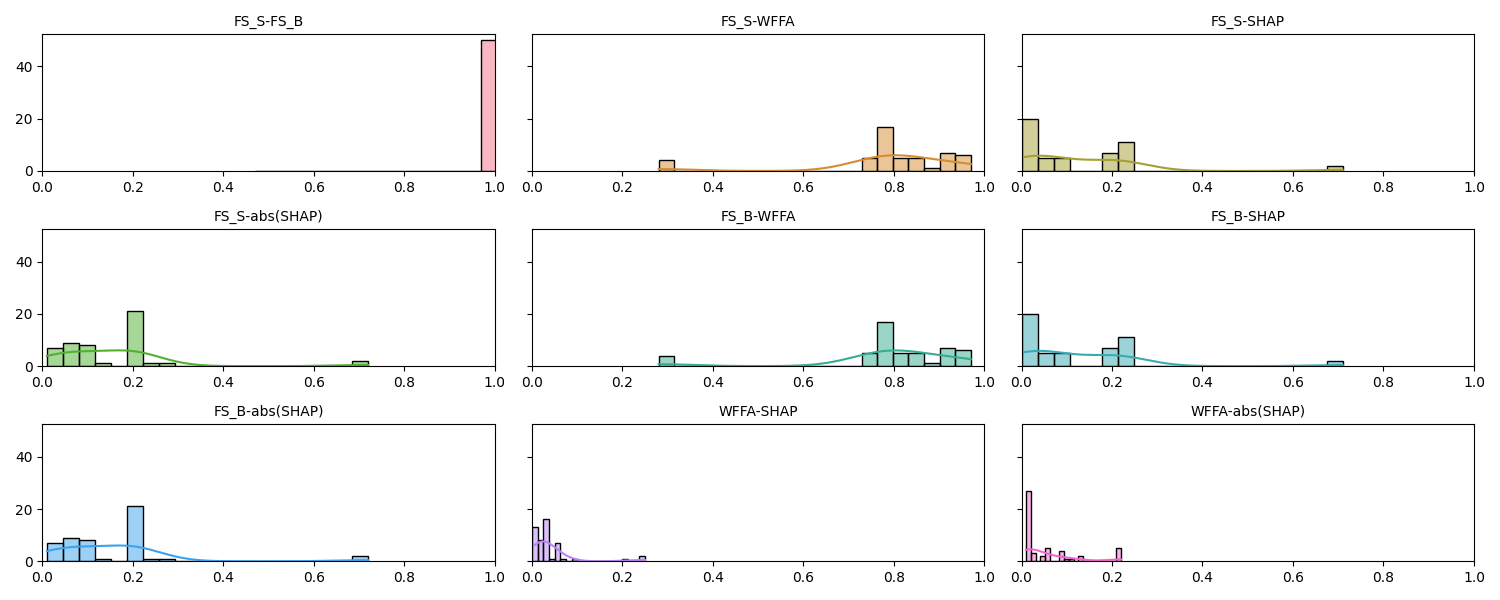}
\caption{Distribution of RBO values for all the tested instances of Qt.}
\label{fig:rbo_distri_qt}
\end{figure*}

\begin{table*}[ht]
\centering

\begin{tabular}{ccrrrrr}
\toprule
                      &           & Adult & COMPAS & Recidivism & Auto-MPG & Boston House Prices \\
\midrule
\multirow{9}{*}{Min}  & $\fsn{S}$-$\fsn{B}$    & 0.79  & 0.93   & 0.92       & 0.97     & 0.97                \\
                      & $\fsn{S}$-WFFA      & 0.1   & 0.22   & 0.26       & 0.12     & 0.1                 \\
                      & $\fsn{S}$-SHAP      & 0.0   & 0.0    & 0.0        & 0.03     & 0.0                 \\
                      & $\fsn{S}$-abs(SHAP) & 0.03  & 0.22   & 0.1        & 0.12     & 0.11                \\
                      & $\fsn{B}$-WFFA      & 0.1   & 0.22   & 0.26       & 0.12     & 0.1                 \\
                      & $\fsn{B}$-SHAP      & 0.0   & 0.0    & 0.0        & 0.03     & 0.0                 \\
                      & $\fsn{B}$-abs(SHAP) & 0.03  & 0.22   & 0.1        & 0.12     & 0.11                \\
                      & WFFA-SHAP        & 0.0   & 0.0    & 0.0        & 0.03     & 0.0                 \\
                      & WFFA-abs(SHAP)   & 0.05  & 0.18   & 0.1        & 0.06     & 0.09                \\
\midrule
\multirow{9}{*}{Max}  & $\fsn{S}$-$\fsn{B}$    & 0.97  & 0.97   & 0.97       & 0.97     & 0.97                \\
                      & $\fsn{S}$-WFFA      & 0.97  & 0.97   & 0.97       & 0.8      & 0.97                \\
                      & $\fsn{S}$-SHAP      & 0.32  & 0.95   & 0.96       & 0.95     & 0.9                 \\
                      & $\fsn{S}$-abs(SHAP) & 0.96  & 0.97   & 0.95       & 0.95     & 0.96                \\
                      & $\fsn{B}$-WFFA      & 0.97  & 0.97   & 0.97       & 0.8      & 0.97                \\
                      & $\fsn{B}$-SHAP      & 0.32  & 0.95   & 0.96       & 0.95     & 0.9                 \\
                      & $\fsn{B}$-abs(SHAP) & 0.96  & 0.97   & 0.95       & 0.95     & 0.96                \\
                      & WFFA-SHAP        & 0.82  & 0.91   & 0.96       & 0.95     & 0.78                \\
                      & WFFA-abs(SHAP)   & 0.88  & 0.95   & 0.95       & 0.97     & 0.91                \\
\midrule
\multirow{9}{*}{Mean} & $\fsn{S}$-$\fsn{B}$    & 0.96  & 0.97   & 0.97       & 0.97     & 0.97                \\
                      & $\fsn{S}$-WFFA      & 0.57  & 0.72   & 0.75       & 0.33     & 0.57                \\
                      & $\fsn{S}$-SHAP      & 0.17  & 0.26   & 0.31       & 0.29     & 0.49                \\
                      & $\fsn{S}$-abs(SHAP) & 0.3   & 0.59   & 0.56       & 0.33     & 0.38                \\
                      & $\fsn{B}$-WFFA      & 0.58  & 0.72   & 0.75       & 0.33     & 0.57                \\
                      & $\fsn{B}$-SHAP      & 0.16  & 0.26   & 0.31       & 0.29     & 0.49                \\
                      & $\fsn{B}$-abs(SHAP) & 0.3   & 0.59   & 0.56       & 0.33     & 0.38                \\
                      & WFFA-SHAP        & 0.25  & 0.4    & 0.31       & 0.33     & 0.27                \\
                      & WFFA-abs(SHAP)   & 0.3   & 0.7    & 0.48       & 0.52     & 0.38              \\
\bottomrule
\end{tabular}

\caption{Rank-biased overlap for the first set of datasets.}
\label{tab:rbo_dt}
\end{table*}

\begin{table*}[ht]
\centering

\begin{tabular}{ccrrr}
\toprule
                      &           & Diabetes & Vehicle & Wine-Recognition \\
\midrule
\multirow{9}{*}{Min}  & $\fsn{S}$-$\fsn{B}$    & 0.47     & 0.09    & 0.79             \\
                      & $\fsn{S}$-WFFA      & 0.1      & 0.01    & 0.0              \\
                      & $\fsn{S}$-SHAP      & 0.01     & 0.01    & 0.01             \\
                      & $\fsn{S}$-abs(SHAP) & 0.1      & 0.0     & 0.02             \\
                      & $\fsn{B}$-WFFA      & 0.1      & 0.01    & 0.0              \\
                      & $\fsn{B}$-SHAP      & 0.01     & 0.01    & 0.01             \\
                      & $\fsn{B}$-abs(SHAP) & 0.1      & 0.0     & 0.02             \\
                      & WFFA-SHAP        & 0.01     & 0.0     & 0.01             \\
                      & WFFA-abs(SHAP)   & 0.09     & 0.0     & 0.01             \\
\midrule
\multirow{9}{*}{Max}  & $\fsn{S}$-$\fsn{B}$    & 0.97     & 0.97    & 0.97             \\
                      & $\fsn{S}$-WFFA      & 0.97     & 0.9     & 0.28             \\
                      & $\fsn{S}$-SHAP      & 0.95     & 0.95    & 0.69             \\
                      & $\fsn{S}$-abs(SHAP) & 0.96     & 0.95    & 0.69             \\
                      & $\fsn{B}$-WFFA      & 0.97     & 0.76    & 0.28             \\
                      & $\fsn{B}$-SHAP      & 0.95     & 0.95    & 0.69             \\
                      & $\fsn{B}$-abs(SHAP) & 0.96     & 0.95    & 0.69             \\
                      & WFFA-SHAP        & 0.95     & 0.91    & 0.22             \\
                      & WFFA-abs(SHAP)   & 0.96     & 0.91    & 0.22             \\
\midrule
\multirow{9}{*}{Mean} & $\fsn{S}$-$\fsn{B}$    & 0.95     & 0.83    & 0.97             \\
                      & $\fsn{S}$-WFFA      & 0.7      & 0.18    & 0.08             \\
                      & $\fsn{S}$-SHAP      & 0.31     & 0.37    & 0.11             \\
                      & $\fsn{S}$-abs(SHAP) & 0.64     & 0.36    & 0.11             \\
                      & $\fsn{B}$-WFFA      & 0.71     & 0.21    & 0.08             \\
                      & $\fsn{B}$-SHAP      & 0.31     & 0.41    & 0.11             \\
                      & $\fsn{B}$-abs(SHAP) & 0.65     & 0.4     & 0.11             \\
                      & WFFA-SHAP        & 0.37     & 0.3     & 0.07             \\
                      & WFFA-abs(SHAP)   & 0.62     & 0.3     & 0.06             \\  
\bottomrule
\end{tabular}

\caption{Rank-biased overlap for the second set of datasets.}
\label{tab:rbo_bt}
\end{table*}

\begin{table*}[ht]
\centering

\begin{tabular}{ccrr}
\toprule
                      &           & Openstack & Qt   \\
\midrule
\multirow{9}{*}{Min}  & $\fsn{S}$-$\fsn{B}$    & 0.97      & 0.97 \\
                      & $\fsn{S}$-WFFA      & 0.08      & 0.28 \\
                      & $\fsn{S}$-SHAP      & 0.0       & 0.0  \\
                      & $\fsn{S}$-abs(SHAP) & 0.01      & 0.01 \\
                      & $\fsn{B}$-WFFA      & 0.08      & 0.28 \\
                      & $\fsn{B}$-SHAP      & 0.0       & 0.0  \\
                      & $\fsn{B}$-abs(SHAP) & 0.01      & 0.01 \\
                      & WFFA-SHAP        & 0.0       & 0.0  \\
                      & WFFA-abs(SHAP)   & 0.01      & 0.01 \\
\midrule
\multirow{9}{*}{Max}  & $\fsn{S}$-$\fsn{B}$    & 0.97      & 0.97 \\
                      & $\fsn{S}$-WFFA      & 0.96      & 0.97 \\
                      & $\fsn{S}$-SHAP      & 0.69      & 0.71 \\
                      & $\fsn{S}$-abs(SHAP) & 0.32      & 0.72 \\
                      & $\fsn{B}$-WFFA      & 0.96      & 0.97 \\
                      & $\fsn{B}$-SHAP      & 0.69      & 0.71 \\
                      & $\fsn{B}$-abs(SHAP) & 0.32      & 0.72 \\
                      & WFFA-SHAP        & 0.88      & 0.25 \\
                      & WFFA-abs(SHAP)   & 0.82      & 0.22 \\
\midrule
\multirow{9}{*}{Mean} & $\fsn{S}$-$\fsn{B}$    & 0.97      & 0.97 \\
                      & $\fsn{S}$-WFFA      & 0.6       & 0.79 \\
                      & $\fsn{S}$-SHAP      & 0.1       & 0.13 \\
                      & $\fsn{S}$-abs(SHAP) & 0.13      & 0.15 \\
                      & $\fsn{B}$-WFFA      & 0.6       & 0.79 \\
                      & $\fsn{B}$-SHAP      & 0.1       & 0.13 \\
                      & $\fsn{B}$-abs(SHAP) & 0.13      & 0.15 \\
                      & WFFA-SHAP        & 0.17      & 0.04 \\
                      & WFFA-abs(SHAP)   & 0.17      & 0.06 \\
\bottomrule
\end{tabular}

\caption{Rank-biased overlap for the third set of datasets.}
\label{tab:rbo_lr}
\end{table*}

\end{document}